%% file: main.tex
\input{_preface}

\usepackage{amsmath, amsthm, amssymb}
\usepackage{booktabs}
\usepackage{colortbl}
\usepackage{csquotes}
\usepackage{inconsolata}
\usepackage{multirow}
\usepackage{listings}
\usepackage[dvipsnames]{xcolor}
\usepackage{comment}
\usepackage{threeparttable}
\usepackage{hyperref}

\definecolor{backgroundColour}{HTML}{FAFAFA}
\definecolor{keywordclr}{HTML}{3F51B5}
\definecolor{commentclr}{HTML}{757575}
\definecolor{stringsclr}{HTML}{279049}
\definecolor{fnctionclr}{rgb}{0.467, 0, 0.533}
\definecolor{builtinclr}{rgb}{0.35, 0, 0.533}
\definecolor{symbolsclr}{rgb}{0.5, 0.25, 0.25}   
\definecolor{numbersclr}{rgb}{0.8, 0.2, 0}
\definecolor{bckgrndclr}{rgb}{0.91, 0.95, 0.95}
\lstdefinestyle{PythonStyle}{
    language=Python,
    backgroundcolor=\color{backgroundColour},
    keywordstyle=\color{keywordclr}\bfseries,
    stringstyle=\color{stringsclr},
    commentstyle=\color{commentclr}\itshape,
    upquote=true,
    basicstyle=\ttfamily\linespread{0.9}\scriptsize,
    breakatwhitespace=false,
    breaklines=true,
    captionpos=b,
    keepspaces=true,
    numbers=left,
    numbersep=5pt,
    numberstyle=\color{commentclr}\ttfamily\tiny,
    showspaces=false,
    showstringspaces=false,
    showtabs=false,
    tabsize=2,
    xleftmargin=1.25em,
    frame=single,
    framexleftmargin=1.25em,
    morekeywords={assert,with,as}
}
\newtheorem{theorem}{Theorem}

\input{custom_commands}
\input{math}

\begin{document}

\maketitle

\begin{abstract}

Lens methods interpret large language models (LLMs) by mapping intermediate activations to the output vocabulary, revealing how next-token predictions develop through the network. Trained lenses remain expensive: affine-translator parameters grow quadratically with model width, while exact, full-vocabulary Kullback--Leibler (KL) training dominates memory. Consequently, prior trained lenses have been applied to models of at most 20B parameters and remain tied to particular component types. We present OmniLens, which applies a single lens family to any model-width activation, whether residual stream, attention, or MLP, and combines two independent scaling techniques. First, low-rank translators make per-lens parameter growth linear in model width and reduce trainable parameters by up to 98.4\%. Second, Subset-KL materializes only selected vocabulary logits: its Top-$k$ mode cuts peak training memory by up to 70\%, while its importance-sampled variant retains unbiased stochastic gradients for the full KL. These savings enable a dense ensemble of 482 lenses for LLaMA-3.3-70B, providing 6$\times$ the coverage of a residual-stream design at the same depth. Model-wide coverage then reveals what single-component lenses cannot: the components where a behavior is most visible need not be those where intervention is most effective, and the most effective interventions lie outside the attention heads examined by prior lens studies. Across three case studies (prompt-injection detection, multi-hop memory injection, and toxicity localization), OmniLens reproduces key published results at substantially lower cost.
\end{abstract}

\input{Sections/1_Introduction}

\input{Sections/2_background}

\input{Sections/3_lora}

\input{Sections/4_subset_kl}

\input{Sections/5_interp_case_studies}

\input{Sections/6_conclusion}


\section*{Acknowledgments}
This research used resources of the Argonne Leadership Computing Facility, which is a U.S. Department of Energy Office of Science User Facility operated under contract DE-AC02-06CH11357.
This work was partially supported by the AuroraGPT project.
MS was supported by the U.S. Department of Energy, Office of Science,
Office of Advanced Scientific Computing Research, Department of Energy Computational Science
Graduate Fellowship under Award Number DE-SC0023112. 

\bibliography{refs}

\clearpage


\appendix

\input{Appendix/lens_taxonomy}
\input{Appendix/experimental_setup}
\input{Appendix/Implementations}
\input{Appendix/AddAblations}
\input{Appendix/kl_estimators}
\input{Appendix/theory}
\input{Appendix/IndexedLogits}
\input{Appendix/Memory_Model}
\input{Appendix/case_studies_appendix}

\end{document}

%% file: _preface.tex
\documentclass[letterpaper]{article} 
\usepackage[preprint]{aaai2027}  
\usepackage{times}  
\usepackage{helvet}  
\usepackage{courier}  
\usepackage[hyphens]{url}  
\usepackage{graphicx} 
\usepackage{natbib}  
\usepackage{caption} 
\usepackage{algorithm}
\usepackage{algorithmic}

\title{Interpreting Language Model Hidden States at Scale}
\author{
    Jordan Pettyjohn\textsuperscript{\rm 1},
    Mansi Sakarvadia\textsuperscript{\rm 1},
    Nathaniel Hudson\textsuperscript{\rm 2,3},
    \\
    Daniel McKenzie\textsuperscript{\rm 4}, 
    Kyle Chard\textsuperscript{\rm 1,3}, 
    Ian Foster\textsuperscript{\rm 1,3}
}
\affiliations{
    \textsuperscript{\rm 1}University of Chicago\\
    \textsuperscript{\rm 2}Illinois Institute of Technology\\
    \textsuperscript{\rm 3}Argonne National Laboratory\\
    \textsuperscript{\rm 4}Colorado School of Mines\\

}

%% file: custom_commands.tex
\colorlet{shaderow1}{blue!10}
\colorlet{shaderow2}{green!15}

\usepackage{pifont}
\newcommand{\cmark}{\textcolor{ForestGreen}{\ding{51}}}
\newcommand{\xmark}{\textcolor{Red}{\ding{55}}}

%% file: math.tex
\DeclareMathOperator{\softmax}{softmax}

\newcommand{\klapprox}{\ensuremath{\mathrm{Top\text{-}}k\mathrm{+IS}}}

%% file: Sections/1_Introduction.tex

\begin{figure}[t]
    \centering
    \includegraphics[width=0.9\linewidth]{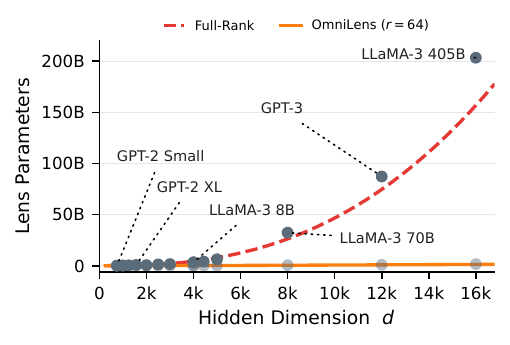}
    \caption{Full-rank lens parameter grows as $\mathcal{O}$(Layer$\times d^2$) exceeding 200B for LLaMA-3-405B.}
    \label{fig:lens_param_explosion}
\end{figure}

\begin{figure*}[t]
    \centering
    \includegraphics[width=0.9\linewidth]{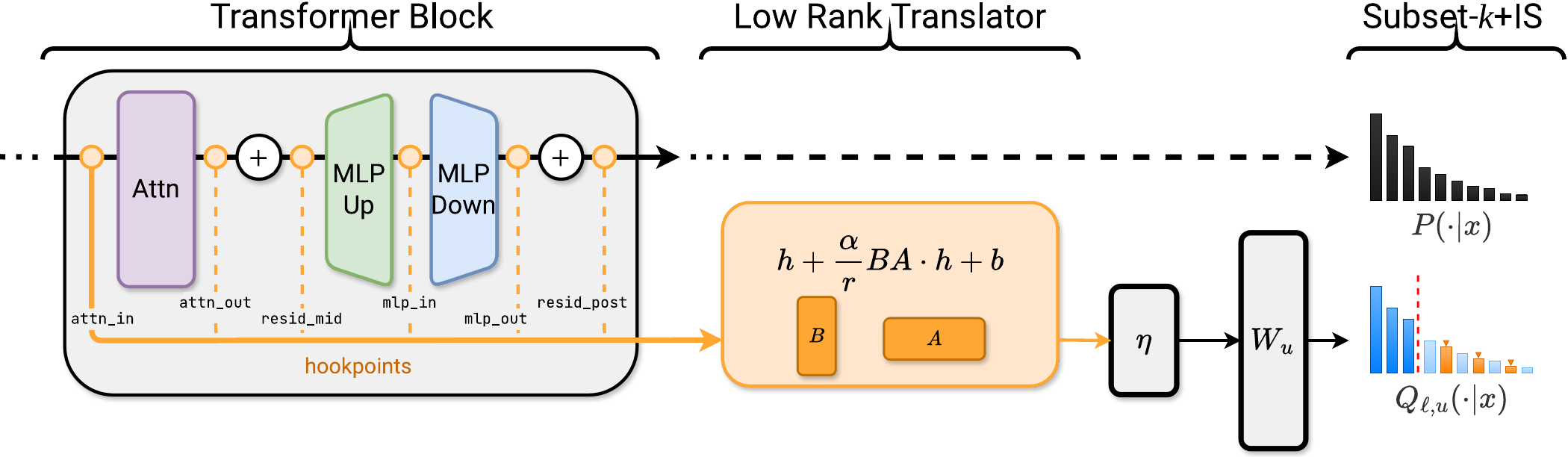}
    \caption{OmniLens at a glance. Hooks are placed at arbitrary user-defined points in the model; each lens applies the low-rank translator of Eq.~\eqref{eq:lora_lens} and is trained against the model's final distribution under the Subset-KL objectives of Section~\ref{sec:subset_kl}.}
    \label{fig:method}
\end{figure*}

\section{Introduction}
\label{sec:introduction}

Understanding how a language model forms its predictions, and where behaviors emerge within its computation, is central to interpreting and controlling it~\citep{orgadinterpretability, shapira2026agentschaos}. A \emph{lens} provides a direct view by decoding intermediate activations into vocabulary distributions, revealing how next-token predictions evolve through the network and where behaviors can be detected or influenced. Lenses can decode many kinds of intermediate activations; we refer to the model component providing input to a lens as its {\em hookpoint}. For example, the \textit{Tuned Lens} reads the residual stream~\citep{belrose2023elicitinglatentpredictionstransformers}, whereas the \textit{Attention Lens} reads individual attention heads~\citep{sakarvadia2023attentionlenstoolmechanistically}.

Unlike concept-specific classifier probes, which require labeled examples for each target~\citep{alain2018understandingintermediatelayersusing,hewitt-manning-2019-structural,belinkov-etal-2017-neural}, lenses can be trained in a self-supervised manner: the model's own final distribution supplies the target, so no curated labels are needed. A single trained lens can therefore support many downstream analyses. This capability is most valuable when applied densely across a model, that is, with multiple hookpoints at every layer. However, training such an ensemble of lenses has remained prohibitively expensive.

Specifically, trained lenses face three challenges when applied at scale.
\emph{Parameters.} A full-rank translator costs $\mathcal{O}(d^2)$ per hookpoint, where $d$ is the model's hidden dimension, so dense coverage of a large model can rival the model itself (Fig.~\ref{fig:lens_param_explosion}). \emph{Memory.} Lens training uses the KL divergence between the model's own final distribution and the lens' distribution as the training objective, which materializes two vocabulary-sized distributions per token. This quickly exhausts available VRAM. \emph{Specialization.} Existing lenses each read one hookpoint type, and each new hookpoint requires a new lens family. 
Together, these challenges limit the application of trained lenses to small models with sparse coverage. 

We address all three challenges with \emph{OmniLens}, {\bf unlocking dense lens coverage for (near-)frontier scale models}. OmniLens is a hookpoint-agnostic framework that applies one lens family to residual, attention, and MLP activations. OmniLens makes dense coverage tractable through two parallel techniques. First, low-rank translators reduce the per-hookpoint parameter count from $\mathcal{O}(d^2)$ to $\mathcal{O}(rd)$ where $r$ is the target rank. Second, a flexible family of approximations to the KL divergence---which we call Subset-KL objectives---materialize only a small subset of the lens' distribution per token, greatly reducing memory requirements. 


We combine these contributions in an open-source framework\footnote{Code: OmniLens (training framework) \url{https://github.com/pettyjohnjn/OmniLens}; Hookbox (hookpoint instrumentation) \url{https://github.com/pettyjohnjn/hookbox}; IndexedLogits (fused CUDA kernel) \url{https://github.com/pettyjohnjn/indexed_logits}; SubsetKL (objectives and estimator) \url{https://github.com/pettyjohnjn/subset-kl}} and reproduce key metrics from prior work at substantially lower cost~\citep{belrose2023elicitinglatentpredictionstransformers,sakarvadia2024memoryinjectionscorrectingmultihop,pettyjohn2025mind}. OmniLens covers six times as many lenses as the tuned lens baseline while costing less, with 90.5\% fewer trainable parameters on LLaMA-3-70B and up to 70\% lower peak training memory on GPT-2. Across three interpretability case studies it matches the application efficacy of existing lens frameworks, and model-wide coverage additionally reveals what single-component lenses cannot: the hookpoints where a behavior is most visible are not necessarily those where intervention is most effective. To test the limits of the approach, we run eight optimization steps on LLaMA-3.1-405B, to our knowledge the first measured demonstration of trained-lens optimization at this scale on existing hardware, establishing that the training path is executable at frontier scale.

%% file: Sections/2_background.tex

\section{Background and Related Work}
\label{sec:background}


We study pretrained autoregressive Transformer language models~\citep{vaswani2017attention}. Let $x$ be a length-$T$ token sequence from vocabulary $V$, and let $P(\cdot|x)$ denote the model's next-token distribution at a given position. The model has hidden dimension $d$ and $L$ layers, with hookpoints exposing residual-stream states, attention and MLP outputs, or individual head outputs, depending on the desired resolution. We write $H_{\ell,u}\in\mathbb{R}^{T\times d_u}$ for the output of component $u$ at layer $\ell$ and $h_{\ell,u}\in\mathbb{R}^{d_u}$ for a single position's activation; lenses apply positionwise, and we suppress the position index. We organize prior work by the three choices that determine a lens method's scalability: where it reads, how it translates the activation, and how its training objective is computed.

\paragraph{Lens formulation and component specialization.}\label{sec:LensBasedInterp}
A lens is an auxiliary decoder mapping an intermediate activation to a distribution over the model's vocabulary. We restrict attention to linear lenses that reuse the model's frozen final normalization $\eta$ (LayerNorm or RMSNorm, applied exactly as in the model's readout) and unembedding $W_U\in\mathbb{R}^{|V|\times d}$:
\begin{equation}
    Q_{\ell, u} ( \cdot | x) = \softmax
    \left(
        W_U\,
        \eta\!\left(\mathcal{L}_{\ell, u}\, h_{\ell, u} + b_{\ell, u}\right)
    \right)
    \label{eq:lens_def}
\end{equation}
\noindent where $\mathcal{L}_{\ell,u}\in \mathbb{R}^{d\times d_u}$ and $b_{\ell,u}\in\mathbb{R}^{d}$ are learned. The logit lens \citep{Nostalgebraist2020LogitLens} applies the model's frozen readout directly to model-width residual-stream states, corresponding to $\mathcal{L}_{\ell,u}=I$ and $b_{\ell,u}=0$. Because this direct readout can be a poor proxy for the final prediction, subsequent methods learn affine translators so that $Q_{\ell,u}(\cdot|x)$ approximates $P(\cdot|x)$ \citep{belrose2023elicitinglatentpredictionstransformers, din2024jump, pal2023future}. These methods target residual streams, while Attention Lens \citep{sakarvadia2023attentionlenstoolmechanistically} learns decoders for individual attention heads. Each construction commits to a component family, so reading a different component requires another lens design; this coupling between translator and component is the specialization bottleneck OmniLens addresses.
The Backward Lens~\citep{katz2024backward} projects gradients rather than activations into vocabulary space, proving that such projections admit low-rank structure.

\paragraph{Parameter-efficient translators.}\label{sec:lens_lineage}
A dense affine translator has $d_ud+d$ learned parameters, and fitting one per component across all layers quickly grows impractical: at the $(6L{+}2)$ hookpoint density we study, the lens set approaches half the parameter count of the base model (Fig.~\ref{fig:lens_param_explosion}), and finer constructions such as the per-head decoders of Attention Lens can exceed the model of study outright. \textit{Low-Rank Adaptation}~(LoRA) \citep{hu2021loralowrankadaptationlarge}---originally used for parameter-efficient fine-tuning---freezes a model's weights and learns a low-rank update to each, cutting trainable parameters by orders of magnitude. 


Low-rank parameterizations have been combined with lenses, yet existing implementations retain the component specialization of their full-rank predecessors. LoRA Lens \citep{pettyjohn2025mind} factorizes per-head Attention Lens decoders as rank-$r$ updates to the frozen unembedding on models up to 8B, while concurrent work \citep{trimigno2026lowrank} trains low-rank residual-stream lenses on models up to 32B. Each applies low rank within a single component family, and neither supplies one translator architecture spanning residual, attention, and MLP components or addresses the vocabulary-side memory cost that limits dense coverage at scale. A full taxonomy of lenses is in Appendix~\ref{appendix:lens_taxonomy}.

\paragraph{Memory-efficient distillation.}\label{sec:Approx_KL}
Most lens frameworks measure the discrepancy between $P(\cdot | x)$ and $Q_{\ell, u} (\cdot | x)$ using the token-level \textit{Kullback-Leibler}~(KL) divergence,
\begin{equation}
\begin{aligned}
    D_{\mathrm{KL}}( P \| Q_{\ell, u} )
        &= \sum_{v \in V} P(v|x) \log\frac{P(v|x)}{Q_{\ell, u}(v|x)}
    \\
        &= \mathbb{E}_{v \sim P(\cdot|x)}\left[\log\frac{P(v|x)}{Q_{\ell, u}(v|x)}\right],
\end{aligned}
\label{eq:KL_definition}
\end{equation}
where we call $P(\cdot | x)$ the \emph{teacher} and $Q_{\ell, u} (\cdot | x)$ the \emph{student}. Materializing all $T|V|$ student logits for a single input $x$ is costly, and numerous approximation schemes avoid it.

We use KL as a distillation loss \citep{hinton2015distillingknowledgeneuralnetwork,sanh2020distilbertdistilledversionbert}. Because the trainable student is the second argument of $D_{\mathrm{KL}}(\cdot \,\|\, \cdot)$ and the expectation runs over the fixed teacher, drawing tokens from $P$ gives a well-behaved and unbiased Monte Carlo estimator. When tokens are drawn from another proposal $R$, importance sampling weights each sampled contribution by $P(v|x)/R(v|x)$ so that its expectation still recovers the original KL \citep{aminibetter2025}. Deterministic Top-$k$ truncation instead scores only the $k$ most probable teacher tokens, a memory-cheap but biased reduction \citep{shao2024deepseekmath}. Appendix~\ref{appendix:kl_estimators} distinguishes this setting from KL regularization under a trainable sampling distribution~\citep{tang2025few}.

\label{sec:related_methods}
The closest sparse-distillation baseline is \textit{Random Sampling Knowledge Distillation}~(RS-KD)~\citep{anshumann2025sparse}, compared against in Section~\ref{sec:is}; our sampled tail builds on importance sampling for large output spaces~\citep{katharopoulos2019samplescreatedequaldeep,pmlr-v80-blanc18a}. These motivate \emph{Subset-KL} (Section~\ref{sec:subset_kl}): biased Top-k truncation and an exact-head, importance-sampled-tail variant with unbiased gradients.

\paragraph{Training-free and complementary readouts.}
The Jacobian lens \citep{gurnee2026jlens} avoids translator training by decoding average local output sensitivity at a hookpoint. It asks what a state locally \emph{represents} under a first-order perturbation, whereas a trained predictive lens is optimized to recover the model's eventual output distribution; its authors observe that tuned lenses can therefore ``skip ahead'' past intermediate representations. PatchScopes~\citep{ghandeharioun2024patchscopes} is likewise training-free, patching hidden states into explanatory prompts so that the model's own generation serves as the readout. These views are complementary, although each remains a per-hookpoint readout. More broadly, lenses are observational decoders and do not by themselves establish causal mechanisms. They complement causal techniques such as circuit discovery and activation patching~\citep{elhage2021mathematical,wang2022interpretabilitywildcircuitindirect,conmy2023automatedcircuitdiscoverymechanistic,meng2022locating}; observational attention analyses~\citep{clark-etal-2019-bert,voita-etal-2019-analyzing,vig-belinkov-2019-analyzing}; and feature-oriented methods including feature visualization, sparse autoencoders that decompose activations into features in superposition, and stochastic parameter decomposition targeting weights rather than activations~\citep{olahFeatureVis,cammarata2021curve,cunningham2023sparseautoencodershighlyinterpretable,templeton2024scaling,sharkey2022taking,bushnaq2025stochastic}. At frontier scale, circuit tracing~\citep{anthropic2025circuits} combines cross-layer transcoders with attribution graphs to map computational structure in production models. Classifier-based probes~\citep{ettinger-etal-2016-probing,conneau2018cramsinglevectorprobing,ivanitskiy2024linearly,kramar2026building} likewise read internal states but test individual concepts using curated labels, whereas lenses target the complete output distribution without task-specific labels.


\paragraph{Synthesis.} Prior work addresses three bottlenecks separately: learned lenses improve fidelity but specialize to one component family, low-rank lenses reduce translator parameters within those families, and sparse distillation reduces vocabulary-side cost without changing where lenses attach. OmniLens combines all three.

%% file: Sections/3_lora.tex
\section{OmniLens: Low-Rank Parameterization}
\label{sec:omnilens}\label{sec:lora_lens}

OmniLens combines low-rank translators, hookpoint-agnostic attachment, and the Subset-KL objectives of Section~\ref{sec:subset_kl}; this section introduces the first two. The same translator parameterization attaches to model-width residual, attention, and MLP activations, generalizing prior hookpoint-specific and low-rank lens constructions (Section~\ref{sec:lens_lineage}). Our implementation supports distributed training and gradient checkpointing; details are in Appendix~\ref{appendix:implementation}.

\paragraph{Low-rank translator.} For an activation $h_{\ell,u}\in\mathbb{R}^{d}$, OmniLens learns a translator $(\mathcal{L}_{\ell,u},b_{\ell,u})$, where
\begin{equation}
    \mathcal{L}_{\ell,u} = I + \frac{\alpha}{r}B_{\ell,u}A_{\ell,u},
    \label{eq:lora_lens}
\end{equation}
$A_{\ell,u}\in\mathbb{R}^{r\times d}$, and 
    $B_{\ell,u}\in\mathbb{R}^{d\times r}$, with fixed scale $\alpha$, so that $\mathcal{L}_{\ell,u}$ is a rank-at-most-$r$ update to the identity.

Each translator contains $2dr+d$ learned parameters, against $d^2+d$ for a dense affine translator. For $d\gg r$ the ratio of learned parameters per hookpoint is approximately $2r/d$, so the savings compound as coverage grows: at $d$ = 4096 and $r$ = 64, six low-rank translators together contain only about 19\% as many parameters as a single dense translator. Low rank can therefore widen coverage across component types while still reducing the total parameter count.

\paragraph{Hookpoint-agnostic attachment.} Existing trained lenses couple their decoder to a particular component type. OmniLens instead treats hookpoint type as a configuration option: each supported hookpoint supplies a model-width activation, after which the same translator parameterization, the same frozen normalization and unembedding, and the same training objective apply unchanged. We cover six residual, attention, and MLP hookpoint types per layer in our experiments (Fig.~\ref{fig:method}), so component types can be compared without designing and training a separate lens family for each. Every hookpoint studied here has width $d_u = d$.


We initialize $A_{\ell,u}$ with Xavier uniform~\citep{glorot2010understanding} and $B_{\ell,u}$ and $b_{\ell,u}$ to zero, so each translator begins as the identity, with $\alpha/r$ controlling the update scale. The translated state is decoded through the model's frozen final normalization $\eta$ and unembedding $W_U$, exactly as in Eq.~\eqref{eq:lens_def}. A translator is therefore only required to be accurate on the activations the model actually produces, rather than on all of $\mathbb{R}^{d}$, and only up to differences that the frozen readout preserves in the output distribution. A full-rank $d \times d$ update supplies more capacity than these two restrictions demand, motivating the low-rank parameterization evaluated below. Fig.~\ref{fig:lens_param_explosion} presents the resulting parameter scaling across model sizes; Appendix~\ref{app:memory} reports the per-model parameter and memory counts underlying these figures.

\paragraph{Rank ablation.}\label{sec:lora_lens_eval} We validate the low-rank parameterization on GPT-2 Small by sweeping $r\in\{1,4,8,16,32,64,128,256,384\}$ against a full-rank tuned-lens baseline after 1000 optimization steps. We measure KL divergence to the teacher, and top-1 agreement, Pearson $\rho$, and Kendall $\tau$ relative to the full-rank lens, following the evaluation protocol of \citet{belrose2023elicitinglatentpredictionstransformers} (Fig.~\ref{fig:rank_tradeoff}; representative points in Table~\ref{tab:rank_summary}). KL measures fidelity to the teacher distribution, while the agreement metrics test whether the low-rank lens preserves the predictions and token rankings of the dense reference. Appendix~\ref{appendix:experimental_setup} gives the setup; Appendix~\ref{appendix:rank_ablation} reports the complete rank and layer breakdowns.

\begin{figure}[t]
    \centering
    \includegraphics[width=0.95\linewidth]{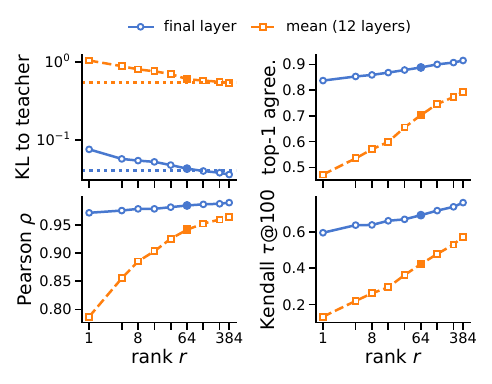}
    \caption{Fidelity-parameter tradeoff on GPT-2 Small at step 1000. Solid lines show the final layer, dashed lines the mean over 12 layers, and dotted horizontal lines represent the full-rank tuned-lens baseline in the KL panel. The filled marker denotes the recommended default, $r=64$; gains diminish at larger ranks. Appendix Table~\ref{tab:lora_fidelity} reports the full sweep.}
    \label{fig:rank_tradeoff}
\end{figure}

\begin{table}[t]
\centering
\caption{Representative points from the rank ablation; the full sweep is
Appendix Table~\ref{tab:lora_fidelity}. Mean metrics average over 12 layers;
KL is to the teacher, agreement metrics are vs.\ the full-rank baseline.
}
\label{tab:rank_summary}
\resizebox{\linewidth}{!}{
\begin{tabular}{lrrrrr}
\toprule
Config. & Params & Mean KL $\downarrow$
        & Top-1 $\uparrow$ & $\rho$ $\uparrow$ & $\tau$ $\uparrow$ \\
\midrule
Full-rank & 100.0\% & 0.543 & 1.000 & 1.000 & 1.000 \\
$r=8$     & 2.1\% & 0.799 & 0.570 & 0.885 & 0.262 \\
$r=32$    & 8.3\% & 0.696 & 0.655 & 0.925 & 0.363 \\
\rowcolor{shaderow1}
$r=64$    & 16.7\% & 0.600 & 0.703 & 0.941 & 0.423 \\
$r=128$   & 33.3\% & 0.574 & 0.746 & 0.952 & 0.479 \\
$r=384$   & 100.0\% & 0.539 & 0.793 & 0.964 & 0.574 \\
\bottomrule
\end{tabular}
}
\end{table}

Fidelity improves rapidly through $r=64$ and only modestly thereafter (Fig.~\ref{fig:rank_tradeoff}). At $r=64$, OmniLens uses 16.7\% as many translator weight parameters as the full-rank baseline while achieving 88.8\% final-layer top-1 agreement (Pearson $\rho=0.984$) and mean $\rho=0.941$; we therefore use $r=64$ as the default. Earlier layers are the hardest to approximate (layer 0: 56.7\% top-1 agreement at $r=64$), and their KL to the teacher remains highest even for the full-rank lens. So, analyses focused on early layers should prefer $r\geq 64$. The parameter-matched run $r{=}384$ slightly outperforms the full-rank baseline on KL, but the difference lies within the baseline's own seed variation (final-layer KL $0.039$--$0.049$ across three seeds; Appendix~\ref{appendix:experimental_setup}); we read this as fidelity preserved, not as low rank being inherently superior. The 8B comparisons in Section~\ref{sec:case_studies} support the same conclusion at scale.


\paragraph{Low rank is a constraint, not a compression.} The deviation from identity learned by the full-rank baseline, $\mathcal{L}_{\ell,u} - I$, is not itself low rank: its best rank-64 approximation retains only $59\%$ of its Frobenius energy on GPT-2 Small, and its numerical rank remains near $d$ throughout training (Fig.~\ref{fig:rank_structure}b). Post-hoc compression accordingly loses fidelity: truncating the full-rank translator to rank 64 raises final-layer KL by $54\%$, while projecting it onto a random 64-dimensional subspace raises KL by more than an order of magnitude. By contrast, a translator trained directly at rank 64 achieves final-layer KL within $7\%$ of the full-rank lens (Fig.~\ref{fig:rank_structure}a). {\em Low-rank training therefore finds a distinct solution of comparable predictive quality rather than merely compressing the dense solution.} Two matrices can differ substantially in Frobenius norm and still induce nearly identical output distributions, because they need only agree on the activations the model produces and only up to differences the frozen readout preserves.

\begin{figure}[t]
    \centering
    \includegraphics[width=0.95\linewidth]{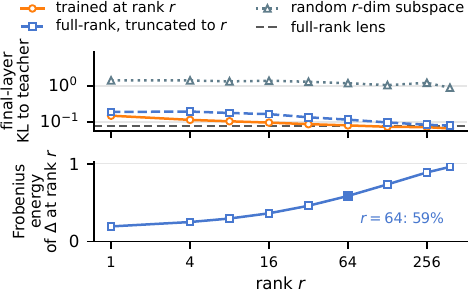}
    \caption{Low rank is a constraint, not a compression (GPT-2 Small, step
    1000, with 131{,}072 held-out Pile tokens).
    \emph{Top:}~Final-layer KL for a translator trained directly at rank $r$, a
    full-rank translator truncated post hoc to rank $r$, and a random
    $r$-dimensional projection of it; the dashed line marks
    the unmodified full-rank lens. 
    \emph{Bottom:}~Frobenius energy of the full-rank deviation captured at rank $r$.
    }
    \label{fig:rank_structure}
\end{figure}

%% file: Sections/4_subset_kl.tex
\section{OmniLens: Subset-KL Training}
\label{sec:subset_kl}

Low-rank translators reduce lens parameters and optimizer state, but they do not narrow the lens output: after translation, every active lens must still score vocabulary items at every token position. Computing the full KL divergence therefore materializes a vocabulary-sized logit tensor for each lens $Q_{\ell,u}$ (Eq.~\eqref{eq:KL_definition}). This activation cost grows with batch size, context length, and vocabulary size, and can dominate memory even when the translator itself is lightweight.

We call objectives that evaluate only selected vocabulary tokens \emph{Subset-KL}, and study two variants with different computational and statistical guarantees. Top-$k$ restricts both distributions to the teacher's most probable tokens and renormalizes within that set: this avoids the full-vocabulary lens projection but changes the objective, and is therefore biased (Section~\ref{sec:topk}). \klapprox\ evaluates the head exactly and importance-samples the remaining vocabulary, retaining the full student normalization and giving unbiased stochastic gradients for the original KL (Section~\ref{sec:is}). Top-$k$ thus prioritizes memory and throughput, whereas \klapprox\ prioritizes fidelity to the full-KL objective. We follow the evaluation protocol of \citet{belrose2023elicitinglatentpredictionstransformers}; Appendix~\ref{appendix:experimental_setup} gives the complete experimental setup.

\subsection{Top-$k$ Truncation}
\label{sec:topk}

Let $\mathcal H$ contain the $k_{\mathrm{head}}$ most probable teacher tokens. Top-$k$ restricts both distributions to $\mathcal H$ and renormalizes them:
\begin{equation}
    D_{\mathrm{Top}\text{-}k}(P\|Q_{\ell,u})
    =
    \sum_{v\in\mathcal H}
    P_{\mathcal H}(v|x)
    \log\frac{P_{\mathcal H}(v|x)}
              {Q_{\mathcal H,\ell,u}(v|x)},
    \label{eq:KL_topk}
\end{equation}
where $P_{\mathcal H}$ and $Q_{\mathcal H,\ell,u}$ denote the restricted, renormalized distributions. Equivalently, Top-$k$ asks the lens to reproduce the teacher's relative preferences among its $k_{\mathrm{head}}$ most probable tokens while ignoring both distributions' mass outside that set. Thus it is not a cheaper implementation of the full KL: truncating the partition changes the training objective.

That change is also Top-$k$'s principal systems advantage: because $Q_{\mathcal H,\ell,u}$ normalizes only over $\mathcal H$, the lens never projects into the complete vocabulary, reducing activation memory \emph{and} projection compute, with $k_{\mathrm{head}}$ as a direct quality--cost control. We evaluate this on GPT-2 Small, where full-KL training remains feasible and supplies a reference: on its residual hookset, $k_{\mathrm{head}}=256$ reduces peak memory from 16.3 to 4.7\,GB and increases throughput by 1.59$\times$, while changing final-layer KL from 0.043 to 0.054.



\subsection{\klapprox: Exact Head, Sampled Tail}
\label{sec:is}
To ameliorate the poor performance of Top-$k$ on early layers, we introduce \klapprox, which keeps the head unnormalized and samples the tail. Let $P_{\mathrm{head}}=\sum_{v\in\mathcal H}P(v|x)$, and draw $t_1,\ldots,t_{k_{\mathrm{tail}}}$ independently from a proposal $R(\cdot|x)$ supported on $V\setminus\mathcal H$. We compute
\begin{equation}
\small
\begin{aligned}
    \widehat{D}_{\text{\klapprox}}(P\|Q_{\ell,u})
    = &
    \sum_{v\in\mathcal H}
    P(v|x)\log\frac{P(v|x)}{Q_{\ell,u}(v|x)} \\
    &+
    \frac{1}{k_{\mathrm{tail}}}
    \sum_{i=1}^{k_{\mathrm{tail}}}
    \frac{P(t_i|x)}{R(t_i|x)}
    \log\frac{P(t_i|x)}{Q_{\ell,u}(t_i|x)}.
\end{aligned}
\label{eq:subset_KL}
\end{equation}
The first line is the exact contribution of the head $\mathcal H$ to the original, untruncated KL. The second estimates the omitted tail sum from sampled tokens, each reweighted by $P/R$ so that in expectation the sampled term reconstructs the complete tail contribution. For any lens-independent proposal positive on the tail, both the estimated objective and its stochastic gradients are therefore unbiased for the full KL. 

\begin{theorem}
Let $t_1, \ldots, t_{k_{\mathrm{tail}}}$ be drawn i.i.d.\ from any proposal $R(\cdot|x)$ supported on $V \setminus \mathcal{H}$ with $R(v|x) > 0$ wherever $P(v|x) > 0$. The \klapprox\ estimator Eq.~\eqref{eq:subset_KL} is unbiased,
\begin{equation*}
\mathbb{E}\left[\hat{D}_{\mathrm{\klapprox}}(P \| Q_{\ell,u})\right] = D_{\mathrm{KL}}(P \| Q_{\ell,u}),
\end{equation*}
and yields unbiased gradients as well: 
\begin{equation*}
    \mathbb{E}\left[\nabla \hat{D}_{\mathrm{\klapprox}}\right] = \nabla D_{\mathrm{KL}}(P \| Q_{\ell,u}). 
\end{equation*}
\label{thm:unbiased_KL_gradients}
\end{theorem}

\vspace{-2ex}

A complete proof for Theorem~\ref{thm:unbiased_KL_gradients} is presented in Appendix~\ref{app:theory}.
%

Here $Q_{\ell,u}$ is the true student distribution rather than one renormalized over the subset. Evaluating it requires the full-vocabulary log-partition $\log\sum_{v} \exp(z_v)$, which we compute exactly without materializing all $|V|$ logits: we project one chunk of the vocabulary at a time, accumulate its log-sum-exp into a running total, and discard the chunk's logits before projecting the next. Subsampling therefore applies only to which KL summands are evaluated, not to the normalizer.

As a default, we sample from the teacher conditioned on the tail, $R(v|x)=P(v|x)/(1-P_{\mathrm{head}})$ for $v\in V\setminus\mathcal H$. This draws tail tokens in proportion to the same probabilities that weight their KL terms and needs no separate proposal model. 
Retaining the head deterministically removes the most heavily weighted terms from the estimator's variance, which is why \klapprox\ improves on pure teacher sampling.
Appendix~\ref{app:theory} gives the sampling procedure and illustrates the estimator (Fig.~\ref{fig:estimator}).

\paragraph{Evaluation.} Table~\ref{tab:estimator_comparison} compares matched token budgets, and Fig.~\ref{fig:subset_kl_comparison} gives the layerwise profiles. Among the subset objectives, Top-$k$ achieves lower final-layer KL and peak memory at both budgets, whereas \klapprox\ achieves lower mean and early-layer KL. At a $256{+}256$ budget, \klapprox\ halves peak memory from $16.3$ to $8.2$\,GB while processing $50.9$k tokens/s. Its throughput remains near $48$--$51$k tokens/s across the evaluated budgets, whereas Top-$k$ slows from $96.5$k tokens/s at $k=256$ to $28.9$k at $k=1024$. Seed replications preserve the mean and early-layer ordering, while showing that part of the final-layer gap is seed variation. Appendix~\ref{app:ablations} reports the complete sweep and per-seed results.

Setting $k_{\mathrm{head}}=0$ reduces Eq.~\eqref{eq:subset_KL} to pure teacher sampling, the RS-KD baseline (Section~\ref{sec:related_methods}). At a total budget of $512$ it reaches final-layer KL $0.194$, against $0.067$ for \klapprox, and its training loss deteriorates as the sample budget grows. Appendix~\ref{app:ablations} gives the full sweep and diagnostics, and Appendix~\ref{appendix:kl_estimators} compares related estimators.

\paragraph{Choosing an objective.} Neither variant dominates. Top-$k$ is preferable under projection cost, throughput, or maximum context length constraints, particularly for analyses that emphasize later layers. \klapprox\ is preferable when lenses are compared throughout the network, or when preserving the semantics of the full-KL objective matters. We therefore default to \klapprox\ for the trained lenses in the interpretability studies, whose deterministic head is what preserves late-layer fidelity (Fig.~\ref{fig:rs_profiles}), and report Top-$k$ as the lower-memory, higher-throughput alternative.


\begin{table}[t]
\centering
\caption{Matched-budget Subset-KL comparison on GPT-2 Small
($r=64$, step $1{,}000$, residual hookset). Early KL averages
layers 0--3; RS denotes teacher sampling
($k_{\mathrm{head}}=0$). Bold marks the best Top-$k$ or
Top-$k$+IS value, and shaded rows share a token budget.
Complete results are in Table~\ref{tab:estimator_sweep}.}
\label{tab:estimator_comparison}

\resizebox{\columnwidth}{!}{%
\begin{tabular}{llrrrrr}
\toprule
Estimator & Config &
  \multicolumn{3}{c}{KL} &
  \multicolumn{2}{c}{Efficiency} \\
\cmidrule(lr){3-5}
\cmidrule(l){6-7}
& & Final & Mean &
  \shortstack{Early\\(0--3)} &
  \shortstack{Tok/s\\($\times 10^3$)} &
  \shortstack{Peak\\GB} \\
\midrule
Full-KL
  & $|V|$
  & 0.043 & 0.600 & 1.019 & 60.8 & 16.3 \\
\midrule
\rowcolor{shaderow1}
Top-$k$
  & $k=512$
  & 0.050 & 0.903 & 1.631 & \textbf{54.6} & \textbf{4.8} \\
\rowcolor{shaderow2}
  & $k=1024$
  & \textbf{0.046} & 0.780 & 1.413 & 28.9 & 5.0 \\
\midrule
\rowcolor{shaderow1}
Top-$k$+IS
  & $256+256$
  & 0.067 & \textbf{0.637} & \textbf{1.062}
  & 50.9 & 8.2 \\
\rowcolor{shaderow2}
  & $512+512$
  & 0.067 & 0.649 & 1.073 & 49.8 & 8.5 \\
\midrule
\rowcolor{shaderow1}
RS
  & $0+512$
  & 0.194 & 0.879 & 1.334 & 50.0 & 8.2 \\
\rowcolor{shaderow2}
  & $0+1024$
  & 0.656 & 1.309 & 1.693 & 47.5 & 8.4 \\
\bottomrule
\end{tabular}%
}
\end{table}

\begin{figure}[t]
    \centering
    \includegraphics[width=\linewidth]{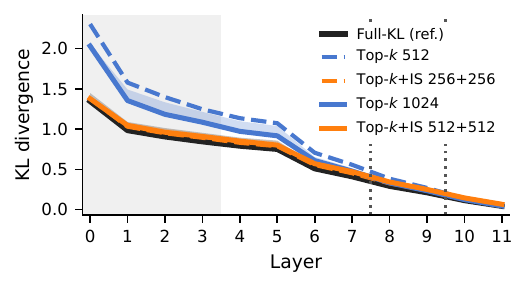}
    \caption{Layerwise KL on GPT-2 Small ($r=64$, step $1{,}000$) at
    matched budgets of 512 (dashed) and 1024 (solid) scored tokens.
    Shading marks layers 0--3; envelopes span three seeds.
    }
    \label{fig:subset_kl_comparison}
\end{figure}

\subsection{A Fused Kernel for Selected Logits}
\label{sec:indexed_logits}

Subset selection introduces a second, implementation-level memory problem that is independent of which objective is used. Both variants require different selected vocabulary rows at each of the $B\!\cdot\!T$ positions in a microbatch. A na\"ive implementation gathers those rows of the unembedding $W_U$ into a $[B\!\cdot\!T,k,d]$ intermediate before multiplying by the translated activation, potentially consuming more memory than the subset objective saves. At $B\!\cdot\!T=8192$, $k=512$, and $d=4096$, this tensor alone occupies $68.7$\,GB in FP32. Our fused CUDA kernel computes each selected logit directly in registers and never forms the gathered tensor, reducing the peak memory to $0.15$\,GB (Appendix~\ref{app:cuda}). 

\paragraph{Combined feasibility.}
Lens parameters and optimizer state persist throughout training and cannot be reduced by microbatching, so low rank is what makes dense 70B coverage possible under our single-device lens placement. Vocabulary-readout activations instead grow with microbatch size, context length, and vocabulary size, and are therefore reducible by microbatching. Subset-KL governs what readout workloads fit at a given budget: at 8B both subset objectives train at 4K context on an A100-40GB, whereas full KL exhausts the device beyond 2K. At 70B, full KL still fits the low-rank configuration at our production microbatch, so Subset-KL is not what makes that individual run possible; it is what determines the context length, batch size, and vocabulary scale reachable without additional hardware, and lets the same framework trade statistical fidelity against memory and throughput. Appendix~\ref{app:memory} gives the complete accounting.

\paragraph{Frontier scale.}
We execute eight optimizer steps of a 126-hookpoint residual-lens stack on LLaMA-3.1-405B-Instruct across $96\times$A100-40GB GPUs. At $r{=}64$ that stack contains $266$M translator parameters, a roughly $127$-fold reduction from the corresponding $33.8$B-parameter dense stack before optimizer state. This run verifies execution of the complete training path at 405B scale; Appendix~\ref{app:memory} gives the configuration, memory measurements, and loss trajectory.

%% file: Sections/5_interp_case_studies.tex
\section{Lens Application Case Studies}
\label{sec:case_studies}

\begin{figure*}[ht]
    \centering
    \includegraphics[width=0.9\linewidth]{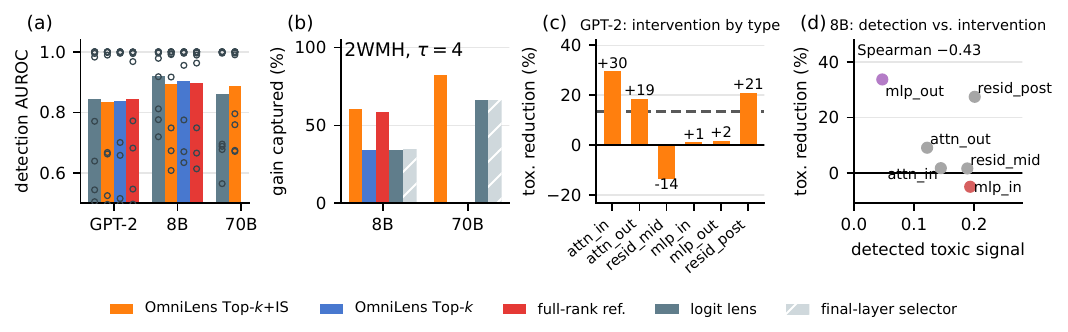}
    \caption{%
    \textbf{Interpretability applications.}
        (a)~Prompt-injection AUROC across ten tasks (bars: means; points: individual tasks).
        (b)~Fraction of achievable 2WMH intervention gain captured at $\tau =4$.
        (c)~GPT-2 toxicity reduction by component; the dashed line marks the prior attention-head-only method.
        (d)~At 8B, detected toxicity and intervention effectiveness are negatively correlated across component types. No full-rank reference is trainable at 70B (Section~\ref{sec:subset_kl}).
    }
    \label{fig:case_summary}
\end{figure*}

We evaluate whether OmniLens preserves the downstream utility of trained lenses by revisiting three applications from the literature: 
\textit{(i)}~prompt-injection detection~\citep{belrose2023elicitinglatentpredictionstransformers}, 
\textit{(ii)}~multi-hop memory injection~\citep{sakarvadia2024memoryinjectionscorrectingmultihop}, 
and 
\textit{(iii)}~toxicity localization~\citep{pettyjohn2025mind}. 
Because the same lens can be attached to residual, attention, and MLP activations, we also extend the toxicity intervention beyond the attention heads examined previously. Unless noted otherwise, results use rank-64 \klapprox\ lenses, with Top-$k$, full-rank, and logit-lens baselines where available. Experimental details/caveats appear in Appendices~\ref{appendix:experimental_setup} and~\ref{app:caveats}.

\paragraph{Prompt-injection detection.}\label{sec:inj_detect}
\citet{belrose2023elicitinglatentpredictionstransformers} show that a prompt injection perturbs the model's intermediate computation before it perturbs the output. Following their protocol, we fit outlier detectors to the layer-by-layer predictions decoded by each lens on clean multiple-choice prompts and test whether they flag prompts carrying an injection attack. On the five tasks where the original study reports near-perfect detection, \klapprox\ attains a mean AUROC of $0.997$ at every scale and matches the full-rank reference within bootstrap uncertainty where that reference exists (Fig.~\ref{fig:case_summary}a). At 70B, where no full-rank reference is available, \klapprox\ detects attacks the logit lens misses on the knowledge tasks (ARC-Easy $0.80$ vs.\ $0.68$; SciQ $0.67$ vs.\ $0.56$).  Protocols and per-task results are in Appendices~\ref{app:inj_detect} and~\ref{app:finegrained}.

\paragraph{Multi-hop factual reasoning.}\label{sec:multihop}
\citet{sakarvadia2024memoryinjectionscorrectingmultihop} attribute multi-hop failures to a recall gap and correct them by \emph{memory injection}: the difference between explicit and implicit intermediate representations is added to the residual stream at a chosen layer. The intervention itself needs no lens; the lens's role is to select \emph{where} to inject. The intervention replicates at scale: on 2WikiMultiHop (2WMH)~\citep{ho2020constructingmultihopqadataset} it raises the model's own $P(\text{answer})$ by 6.6$\times$ at 8B and 5.1$\times$ at 70B. At a deliberately strong intervention scale ($\tau{=}4$), where the causal optimum moves to an interior layer, the \klapprox-selected layer captures $61\%$ (8B) and $82\%$ (70B) of the maximum achievable gain, compared with $34\%$ and $66\%$ for final-layer injection (Fig.~\ref{fig:case_summary}b). \klapprox\ is the only evaluated lens that selects a middle layer and remains near-optimal across all three training seeds. Selection criteria, layer-by-$\tau$ sweeps, controls, and distortion measurements are in Appendix~\ref{app:injection}.

\paragraph{Toxicity localization and intervention.}\label{sec:dart}\label{sec:audit}
Following \citet{pettyjohn2025mind}, we rank attention heads by the prevalence of toxic tokens in their lens predictions, then subtract an unembedding-derived toxicity direction at the selected heads. At 8B, OmniLens reproduces the concentration previously observed on GPT-2: the top five heads carry approximately 35\% of the detected toxic signal, and the strongest head is stable across lens variants and training seeds. Extending the intervention across all six hookpoint types, however, changes the conclusion. The most effective targets lie outside the attention heads examined previously at both GPT-2 and 8B and differ between models (Fig.~\ref{fig:case_summary}c,d). At 8B, detection and intervention rankings are negatively correlated (Spearman $-0.43$): MLP outputs show the weakest detected toxicity but produce the largest reduction when modified, whereas MLP inputs score highly under detection but increase toxicity when modified. Where a behavior is most visible, then, need not be where it is most causally actionable. Complete rankings, intervention sweeps, and the 70B analysis appear in Appendix~\ref{app:dart}. Together, these studies show that OmniLens preserves the established applications of trained lenses while enabling model-wide comparisons that component-specific lens families cannot make.

%% file: Sections/6_conclusion.tex
\section{Discussion and Conclusion}
\label{sec:discussion}


\paragraph{Conclusion.}
OmniLens makes dense trained-lens coverage practical at previously inaccessible model scales by combining low-rank translators, memory-efficient vocabulary-subset objectives, and one framework for residual, attention, and MLP activations. We train a 482-hookpoint lens ensemble for LLaMA-3.3-70B and demonstrate training for eight optimizer steps on LLaMA-3.1-405B. Across three established interpretability applications, the resulting lenses preserve key reference results while enabling comparisons that component-specific methods cannot make, including the finding that the hookpoints where a behavior is most visible and those where intervention is most effective can differ. OmniLens thus moves trained lenses from sparse, specialized readouts toward a model-wide interpretability instrument.

\paragraph{Broader impact.}

Democratizing model-wide interpretability at modern LLM scales matters for AI safety, governance, and research equity. OmniLens enables researchers to apply the same trained-lens framework to anomaly detection, causal hookpoint selection, and behavioral localization throughout an architecture rather than at a few preselected components. Prior work demonstrates how such localization can guide targeted interventions on safety-relevant behavior, including toxicity~\citep{lee2024mechanistic}, refusal~\citep{arditi2024refusal}, and honesty~\citep{zou2025representation}, as well as weight-level editing of specific factual associations~\citep{meng2022locating,meng2023memit}. The same capabilities could expose model vulnerabilities or manipulation points, a dual-use risk shared by interpretability tools generally.

\paragraph{Limitations and future work.}\label{sec:Limitations}
Our experiments cover only components with $d_u = d$. The parameterization extends to narrower components, such as individual attention heads, by adding a map into the residual space, which we do not evaluate here. The two vocabulary-subset objectives make different tradeoffs: Top-$k$ avoids the full-vocabulary projection but is biased, whereas \klapprox\ retains unbiased gradients while still computing the full normalization. These boundaries motivate learned input maps, more efficient unbiased normalization, and converged evaluations beyond 70B.

%% file: Appendix/lens_taxonomy.tex
\section{Lens Taxonomy}
\label{appendix:lens_taxonomy}

Table~\ref{tab:lens_taxonomy} organizes the published (linear) lens constructions as discussed in Section~\ref{sec:LensBasedInterp}. We note there are two linear lens parameterizations in the literature. The most common is to view the lens as a component-to-residual stream map---i.e. $\mathcal{L}_{\ell,u}\in\mathbb{R}^{d\times d_u}$, where $d_u$ is the width of component $u$ (equal to the model width $d$ for every hookpoint we train)---whence composing with the unembedding matrix maps to vocabulary space: 
\begin{equation*}
    Q_{\ell, u} ( \cdot | x) = \softmax
    \left(
        W_U\,
        \eta\!\left(\mathcal{L}_{\ell, u}\, h_{\ell, u} + b_{\ell, u}\right)
    \right).
\end{equation*}
This is the formulation presented in Section~\ref{sec:LensBasedInterp}. In the second parameterization the lens maps directly to vocabulary space. That is, $\mathcal{L}_{\ell,u}\in\mathbb{R}^{|V|\times d_u}$ and 
\begin{equation}
    Q_{\ell, u} ( \cdot | x) = \softmax
    \left(\mathcal{L}_{\ell, u}\, h_{\ell, u} + b_{\ell, u}
    \right).
    \label{eq:attn_lens_param}
\end{equation}
We mark the methods in Table~\ref{tab:lens_taxonomy} employing the parameterization of Eq.~\eqref{eq:attn_lens_param} with a dagger ($\dagger$).
\begin{table*}[ht]
\centering
\caption{Published lens constructions. All lenses use the parameterization Eq.~\eqref{eq:lens_def} unless marked by $\dagger$, in which case they use the parameterization of Eq.~\eqref{eq:attn_lens_param}.}
\label{tab:lens_taxonomy}
\resizebox{\linewidth}{!}{%
\begin{tabular}{lll}
\toprule
Method & Component $u$ & Translator $\mathcal{L}_{\ell,u}$ \\
\midrule
Logit lens \citep{Nostalgebraist2020LogitLens} & residual stream & $I$ (frozen) \\
Tuned lens \citep{belrose2023elicitinglatentpredictionstransformers} & residual stream & full-rank affine, per layer \\
Future lens \citep{pal2023future} & residual stream & full-rank affine, per offset \\
Jump-to-conclusions \citep{din2024jump} & residual stream & full-rank linear, cross-layer \\
Attention lens \citep{sakarvadia2023attentionlenstoolmechanistically}$\dagger$ & attention-head outputs & full-rank affine, per layer \\
Jacobian lens \citep{gurnee2026jlens} & residual stream (mid-layers) & averaged Jacobian (frozen) \\
LoRA Lens \citep{pettyjohn2025mind}$\dagger$ & attention-head outputs & low-rank unembedding update \\
Low-rank lens \citep{trimigno2026lowrank} & residual stream & low-rank residual, full-KL trained \\
OmniLens (ours) & any component & low-rank residual Eq.~\eqref{eq:lora_lens} \\
\bottomrule
\end{tabular}}
\end{table*}

Two axes distinguish the entries of Table~\ref{tab:lens_taxonomy}. The first is \emph{where} the lens reads: prior methods each fix one component type, while the hookpoint-agnostic parameterization of Section~\ref{sec:lora_lens} covers all of them under one training run. The second is \emph{how much} the translator costs: the earlier trained entries use a full-rank $d \times d$ map per hookpoint, the two low-rank predecessors reduce that cost at a single fixed component, and our identity-residual translator Eq.~\eqref{eq:lora_lens} carries the low-rank cost to every component type.

%% file: Appendix/experimental_setup.tex
\section{Experimental Setup and Implementation}
\label{appendix:experimental_setup}


\paragraph{Models.} GPT-2 Small (124M; 12 layers, $d{=}768$, $|V|{=}50{,}257$; HF \texttt{openai-community/gpt2}), LLaMA-3-8B (32 layers, $d{=}4096$, $|V|{=}128{,}256$; HF \texttt{meta-llama/Meta-Llama-3-8B-Instruct}), and LLaMA-3.3-70B (80 layers, $d{=}8192$, $|V|{=}128{,}256$; HF \texttt{meta-llama/Llama-3.3-70B-Instruct}; abbreviated LLaMA-3-70B in the main text). All models are frozen bf16 checkpoints; the LLaMA models are the instruct-tuned variants, while GPT-2 is the base pretrained model. The 405B feasibility run (Section~\ref{sec:subset_kl}) additionally uses LLaMA-3.1-405B-Instruct (126 layers, $d{=}16{,}384$, $|V|{=}128{,}256$; HF \texttt{meta-llama/Llama-3.1-405B-Instruct}), likewise a frozen bf16 instruct checkpoint.

\paragraph{Data.} All lenses were trained on \texttt{The Pile}~\citep{gao2020pile}, tokenized into non-overlapping chunks of the training sequence length (stride equals sequence length). No additional preprocessing or filtering was applied.\footnote{\label{fn:corpus}A preliminary swap to an alternative corpus produced negligible differences in final-layer and layerwise KL, suggesting lens quality is determined primarily by the model's representations rather than the training distribution. We did not study this systematically.}

\paragraph{Optimization.} We used AdamW with default moments $(\beta_1, \beta_2) = (0.9, 0.999)$, weight decay $0$, gradient clipping at norm $1.0$, and learning rate $10^{-3}$ with cosine decay to zero. Optimizer steps target an effective batch of $2^{18} = 262{,}144$ tokens, accumulated over an integer number of full microbatches (a step never ends mid-microbatch; when the token target is not divisible by the per-microstep token count, the trainer rounds \emph{up} to the next full microbatch). The realized effective batches are therefore $262{,}144$ tokens for GPT-2 ($8$ microsteps of $32 \times 1024$ on a single A100-40GB), $327{,}680$ for LLaMA-3-8B ($4$ microsteps of $2 \times 1024$ per rank, DDP over $10$ nodes $= 40$ ranks), and $393{,}216$ for LLaMA-3-70B ($2$ microsteps of $2 \times 1024$ per rank over $96$ FSDP ranks on $24$ nodes of $4\times$A100-40GB). GPT-2 trains for $1{,}000$ steps with no warmup; the LLaMA lens sets train for $250$ steps with $50$ warmup steps. All runs use bf16 mixed precision with fp32 loss reductions and log-partition computation, and training seed $0$ unless stated otherwise (see \emph{Checkpoint selection and seeds}).

\paragraph{Software.} Lens training uses Python 3.10, PyTorch 2.10 (CUDA 12.8), Transformers 5.3, and Datasets 4.6 on Linux, on the $4\times$A100-40\,GB nodes described above. KL evaluation reuses the upstream tuned-lens loop, which ships with the released OmniLens code alongside per-run configurations (\texttt{run\_config.csv}). Hyperparameter exploration is reported where it occurred: rank $r$ (nine values; Table~\ref{tab:lora_fidelity}), estimator budgets $k$ and $(k_{\mathrm{head}}, k_{\mathrm{tail}})$ (Table~\ref{tab:estimator_sweep}), and training length (convergence diagnostics; Appendix~\ref{app:ablations}) were swept with the selection criteria stated alongside each; optimizer settings follow standard tuned-lens practice and were not tuned.

\paragraph{Lenses and objectives.} LoRA lenses use $r{=}64$ by default with $\alpha{=}r$ (unit scaling). $A$ was initialized as Xavier uniform \cite{glorot2010understanding} while $B$ and $b$ were initialized to zero, so that the lens is the identity map at initialization. 

{\em Estimator budgets:} the GPT-2 sweeps are as listed in Tables~\ref{tab:estimator_sweep} and~\ref{tab:rank_sensitivity}; the production LLaMA lenses use Top-$k$ with $k{=}512$ and \klapprox\ with $(k_{\mathrm{head}}, k_{\mathrm{tail}}) = (512, 1024)$ under the teacher-tail proposal. Full-KL runs use chunked evaluation with chunk size $4{,}096$. 

{\em Hookpoints:} Residual places one translator per layer while expanded places six hookpoint types per layer plus embedding and final-norm readouts ($6L{+}2$: 74 / 194 / 482 hookpoints at GPT-2 / LLaMA-3-8B / LLaMA-3-70B). Dense coverage costs about the same as a standard one-per-layer full-rank stack at GPT-2 and far less at scale (GPT-2: 7.3M parameters vs.\ 7.1M full-rank residual-only; 8B: 102M vs.\ 537M; 70B: 509M vs.\ 5.4B).

{\em Baselines:} We select the strongest set of baselines that are trainable for a given model size. For GPT-2 every variant is trainable, including the full-rank full-KL reference. For LLaMA-3-8B a multi-node full-rank reference exists on the residual hookset; at 70B no full-rank variant is trainable (optimizer state $388$\,GB; Figure~\ref{fig:needboth_opt}) and the dense LoRA${+}$Subset-KL stack trains at a measured $34.7$\,GB per GPU. 

{\em Training costs:} training one production OmniLens stack (expanded hookset, LoRA $r{=}64$, Subset-KL) costs a few GPU-hours at GPT-2, ${\approx}20$ node-hours at 8B, and ${\approx}360$ node-hours at 70B on $4\times$A100-40\,GB nodes.

\paragraph{Checkpoint selection and seeds.}GPT-2 comparisons use step $1{,}000$; LLaMA lenses use the end of the annealed schedule (step $250$); Appendix~\ref{app:ablations} reports convergence diagnostics supporting both budgets. Headline tables report training seed $0$; two seed studies quantify training-seed variability under otherwise identical recipes. At GPT-2, the full-rank baseline, LoRA full-KL, Top-$k$ ($k{=}1024$), \klapprox\ ($512{+}512$), and RS ($k_{\mathrm{head}}{=}0$) configurations were each retrained at seeds $1$--$2$. Mean and early-layer KL values are stable for every estimator (sd $\le 0.04$ nats), while final-layer KL sd ranges from $0.0007$ (Top-$k$) to $0.012$ (\klapprox). The full-rank baseline's own final KL spans $0.039$--$0.049$. 

For LLaMA-3-8B, the three lens variants used in the case studies---\klapprox, Top-$k$, and the full-rank residual reference---were each retrained at seeds $1$--$2$ under the identical 250-step schedule; the resulting DART and injection-selection stability is reported in Section~\ref{sec:case_studies}. The detection experiments ensemble the isolation forest over five seeds and report bootstrap confidence intervals (Appendix~\ref{app:inj_detect}). Full per-run configurations (\texttt{run\_config.csv}) ship with the released OmniLens code.

%% file: Appendix/Implementations.tex
\subsection{Implementation Details}
\label{appendix:implementation}

\subsubsection{Hookbox System Architecture}

The hookbox system provides a component-agnostic interface for attaching lenses to arbitrary intermediate activations during the forward pass. Hooks toggle independently at any granularity (whole blocks, individual heads, or projections within a head), so one instrumented model serves coarse, fine-grained, and targeted analyses (Figure~\ref{fig:hookbox}). This section details the key design decisions and implementation strategies.

\begin{figure*}[t]
    \centering
    \includegraphics[width=0.95\linewidth]{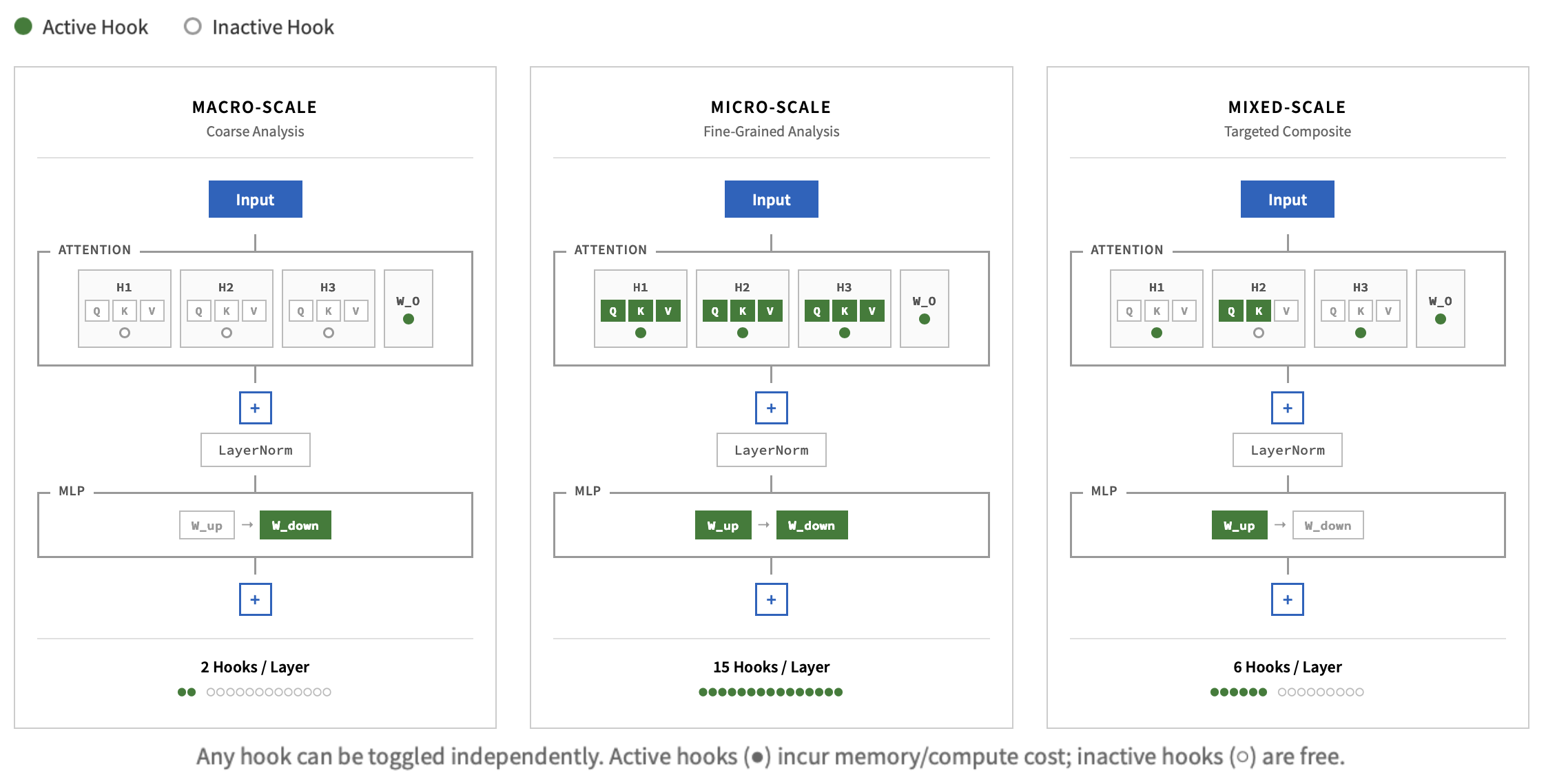}
\caption{Dynamic hook configurations in \texttt{hookbox}. Any hook can be toggled independently: a macro-scale configuration (left) reads only block outputs, a micro-scale configuration (center) instruments every projection of every head, and mixed-scale configurations (right) target specific components. Active hooks ($\bullet$) incur memory and compute cost; inactive hooks ($\circ$) are free.}
    \label{fig:hookbox}
\end{figure*}

\paragraph{Distributed Model Handling}

Modern large language models are typically wrapped in distributed training frameworks (DDP, FSDP, DeepSpeed) that shard parameters and optimizer states across devices. These wrappers introduce additional module hierarchy that must be unwrapped to access the underlying model architecture.

Our implementation automatically detects and unwraps these wrappers using a recursive traversal:
\begin{verbatim}
def unwrap_model(model):
    """Recursively unwrap distributed wrappers."""
    while hasattr(model, 'module'):
        model = model.module
    if hasattr(model, '_fsdp_wrapped_module'):
        model = model._fsdp_wrapped_module
    return model
\end{verbatim}

For FSDP models using ZeRO-3 (full parameter sharding), individual modules may be sharded across devices. We implement an activation gathering mechanism that consolidates sharded tensors before passing to the lens:
\begin{verbatim}
@torch.no_grad()
def gather_from_shards(tensor, process_group):
    """Gather tensor from FSDP shards."""
    world_size = dist.get_world_size(process_group)
    gather_list = [torch.zeros_like(tensor) 
                   for _ in range(world_size)]
    dist.all_gather(gather_list, tensor, 
                    group=process_group)
    return torch.cat(gather_list, dim=0)
\end{verbatim}

\paragraph{Activation Checkpointing Compatibility}

Gradient checkpointing (activation recomputation) reduces memory by discarding intermediate activations during the forward pass and recomputing them during the backward pass. However, this creates a challenge for hooks: they are invoked twice (once during initial forward, once during recomputation), potentially double-counting activations.

Recomputation passes are marked with an explicit context-manager flag that the checkpointed forward is wrapped in:
\begin{verbatim}
class CheckpointingState:
    _is_recomputing: bool = False

    @classmethod
    def is_recomputing(cls):
        return cls._is_recomputing

    @classmethod
    @contextmanager
    def recomputation_context(cls):
        old = cls._is_recomputing
        cls._is_recomputing = True
        try:
            yield
        finally:
            cls._is_recomputing = old
\end{verbatim}

Hooks skip execution during recomputation to avoid duplicate processing:
\begin{verbatim}
def hook_fn(module, input, output):
    if CheckpointingState.is_recomputing():
        return output  # Skip during recompute
    activations = process_activations(output)
    return output
\end{verbatim}

\paragraph{Hook Registration via Predicates}

Rather than hardcoding module names, we allow users to specify attachment points via predicate functions over (name, module) pairs. This provides flexibility and architectural agnostic attachment:

\begin{verbatim}
def register_hooks(model, predicate, hook_fn):
    """Register hooks on modules matching 
    predicate."""
    handles = []
    for name, module in model.named_modules():
        if predicate(name, module):
            handle = module.register_forward_hook(
            hook_fn
            )
            handles.append(handle)
    return handles

# Example usage
pred = lambda n, m: 'layer' in n and 'output' in n
handles = register_hooks(model, pred, my_hook)
\end{verbatim}

Common predicates include:
\begin{itemize}
\item Residual streams: \texttt{lambda n,m: 'layer.\{\}.output' in n}
\item All attention heads: \texttt{lambda n,m: isinstance(m, AttentionHead)}
\item Specific layers: \texttt{lambda n,m: n in ['layer.0', 'layer.10', 'layer.20']}
\item MLP sublayers: \texttt{lambda n,m: 'mlp' in n and 'post' in n}
\end{itemize}

\subsubsection{Numerical Stability Considerations}

\paragraph{Mixed Precision Training}

We use BF16 mixed precision to reduce memory and accelerate training. However, certain operations require FP32 for numerical stability:
\begin{itemize}
\item Logit computation accumulates in FP32 (fused kernel)
\item KL divergence uses log-space computations to avoid underflow
\item Gradient clipping operates on FP32 master weights
\item Normalization layers maintain FP32 running statistics
\end{itemize}

\paragraph{Softmax and Log-Softmax Stability}

Computing KL divergence requires both $\log P$ and $\log Q$. We use the log-sum-exp trick to prevent overflow/underflow:
\begin{verbatim}
def stable_log_softmax(logits):
    """Numerically stable log-softmax."""
    max_logit = logits.max(dim=-1, keepdim=True)[0]
    shifted = logits - max_logit
    return shifted - torch.log(
        torch.exp(shifted).sum(dim=-1, keepdim=True)
    )
\end{verbatim}

The two Subset-KL modes differ here. Top-$k$ renormalizes over the selected subset only (a deliberately truncated objective). The \klapprox\ mode instead uses the \emph{exact} full-vocabulary log-partition: the student's full logit row is produced by a single fused matmul, reduced immediately to its logsumexp, and only the selected subset of logits is retained ($\mathcal{O}(N \cdot V)$ transient per site, a few hundred MB at our configurations, freed by layer-wise backward). Gradients therefore flow to every logit through the partition term, which is what makes the estimator's gradients exactly unbiased (Theorem~\ref{thm:unbiased_KL_gradients}); we verify this numerically against full-KL autograd.

\paragraph{Numerical Guards}

Teacher and student log-probabilities are computed directly with fp32
\texttt{log\_softmax}; probabilities are never reconstructed by exponentiation
and re-floored, so the teacher distribution entering the loss is unmodified.
The tail mass is computed by directly summing the teacher's tail
probabilities in fp32 (not as $1-P_{\mathrm{head}}$, which would lose
precision when the head mass is close to one). Two guards exist in the
sampling path: that tail-mass sum is floored at $10^{-12}$ where it appears
as a normalizing denominator (preventing division by zero only in the
degenerate case where the head captures all numerical mass), and proposal
probabilities are floored at the fp32 subnormal boundary ($10^{-45}$) before
their logarithm. Neither guard biases the estimator: a token with zero
proposal probability can never be drawn by multinomial sampling, so the floor
is never active at a sampled index and the importance weights of
Eq.~\eqref{eq:subset_KL} are unaffected. (This argument concerns
numerics only; the support condition of
Theorem~\ref{thm:unbiased_KL_gradients}, $R(v|x)>0$ wherever $P(v|x)>0$ on
the tail, must hold for the estimator itself, and the teacher-tail default
satisfies it by construction.)

\subsubsection{Distributed Training Configuration}

Table~\ref{tab:distributed_config} summarizes the distributed training strategy used at each model scale. Full configuration details are available in \texttt{OmniLens/configs/distributed/}

\begin{table}[h]
\centering
\caption{Distributed training configuration for different model scales.}
\label{tab:distributed_config}
\small
    \begin{tabular}{llp{0.4\linewidth}}
    \toprule
    Model & Strategy & Configuration \\
    \midrule
    GPT-2 (124M) & Single GPU & 1$\times$ A100-40GB, 8 microsteps/step \\
    \midrule
    LLaMA-3-8B & DDP & 10 nodes $\times$ 4 A100-40GB, 4 microsteps/step \\
    \midrule
    LLaMA-3.3-70B & FSDP & 24 nodes $\times$ 4 A100-40GB, full parameter sharding (ZeRO-3), 2 microsteps/step\\
    \bottomrule
    \end{tabular}
\end{table}

%% file: Appendix/AddAblations.tex
\section{Additional Ablations}
\label{app:ablations}

\subsection{Rank Ablation: Full Layerwise Results}
\label{appendix:rank_ablation}

Table~\ref{tab:lora_fidelity} is the complete rank sweep summarized by
Figure~\ref{fig:rank_tradeoff} and Table~\ref{tab:rank_summary} in
Section~\ref{sec:lora_lens}.
Figures~\ref{fig:rank_ablation_appendix}--\ref{fig:rank_rankcorr_appendix}
show the full layerwise heatmaps for all metrics reported in
Section~\ref{sec:lora_lens}. All metrics show the same qualitative pattern:
at the late layers, a small rank already closes most of the gap to the
full-rank reference, and increasing $r$ closes it further; layer 0's gap
shrinks far more slowly, so the earliest layers remain the binding
constraint at every rank.

\begin{table*}[t]
    \centering
    \caption{Full rank ablation on GPT-2 Small over 131{,}072 Pile test tokens.
    KL is to the teacher; Top-1, Pearson $\rho$ (8{,}192 positions/layer), and
    Kendall $\tau$@100 (512 positions/layer) are vs.\ the full-rank tuned-lens
    baseline. \emph{Final}/\emph{mean} = final layer vs.\ average over 12 layers.
    \colorbox{shaderow1}{Shaded} row ($r{=}64$) is the recommended default;
    \textbf{bold} marks KL values that numerically exceed the baseline (no
    statistical claim: the baseline's own final KL varies $0.039$--$0.049$
    across training seeds, wider than these gaps).}
    \label{tab:lora_fidelity}
    \resizebox{\textwidth}{!}{%
    \begin{tabular}{lrrrrrrrrrr}
        \toprule
        Config. & Params
                & Final KL & Mean KL
                & Final Top-1 & Mean Top-1
                & Final $\rho$ & Mean $\rho$
                & Final $\tau$ & Mean $\tau$ \\
        \midrule
        Full-rank & 100\%
                  & 0.0411 & 0.543
                  & 1.000 & 1.000
                  & 1.000 & 1.000
                  & 1.000 & 1.000 \\
        \midrule
        $r = 1$   & 0.3\%
                  & 0.0760 & 1.037
                  & 0.837 & 0.472
                  & 0.971 & 0.787
                  & 0.596 & 0.133 \\
        $r = 4$   & 1.0\%
                  & 0.0574 & 0.878
                  & 0.853 & 0.535
                  & 0.975 & 0.855
                  & 0.638 & 0.220 \\
        $r = 8$   & 2.1\%
                  & 0.0546 & 0.799
                  & 0.859 & 0.570
                  & 0.978 & 0.885
                  & 0.639 & 0.262 \\
        $r = 16$  & 4.2\%
                  & 0.0526 & 0.755
                  & 0.868 & 0.599
                  & 0.978 & 0.903
                  & 0.662 & 0.295 \\
        $r = 32$  & 8.3\%
                  & 0.0479 & 0.696
                  & 0.878 & 0.655
                  & 0.981 & 0.925
                  & 0.670 & 0.363 \\
        \rowcolor{shaderow1}
        $r = 64$  & 16.7\%
                  & 0.0431 & 0.600
                  & 0.888 & 0.703
                  & 0.984 & 0.941
                  & 0.693 & 0.423 \\
        $r = 128$ & 33.3\%
                  & \textbf{0.0402} & 0.574
                  & 0.900 & 0.746
                  & 0.986 & 0.952
                  & 0.718 & 0.479 \\
        $r = 256$ & 66.7\%
                  & \textbf{0.0382} & 0.555
                  & 0.907 & 0.773
                  & 0.987 & 0.959
                  & 0.739 & 0.531 \\
        $r = 384$ & 100\%
                  & \textbf{0.0362} & \textbf{0.539}
                  & 0.915 & 0.793
                  & 0.989 & 0.964
                  & 0.762 & 0.574 \\
        \bottomrule
    \end{tabular}}
\end{table*}

\begin{figure}[t]
    \centering
    \includegraphics[width=\linewidth]{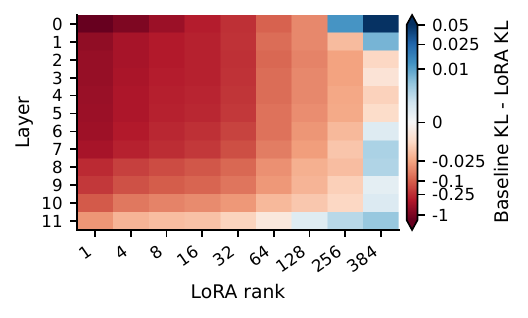}
    \caption{%
        Layerwise KL difference (baseline KL minus LoRA KL) by layer and
        rank, GPT-2 Small. Red indicates LoRA exceeds the full-rank
        baseline; blue indicates LoRA improves upon it.
    }
    \label{fig:rank_ablation_appendix}
\end{figure}

\begin{figure*}[t]
    \centering
    \includegraphics[width=\linewidth]{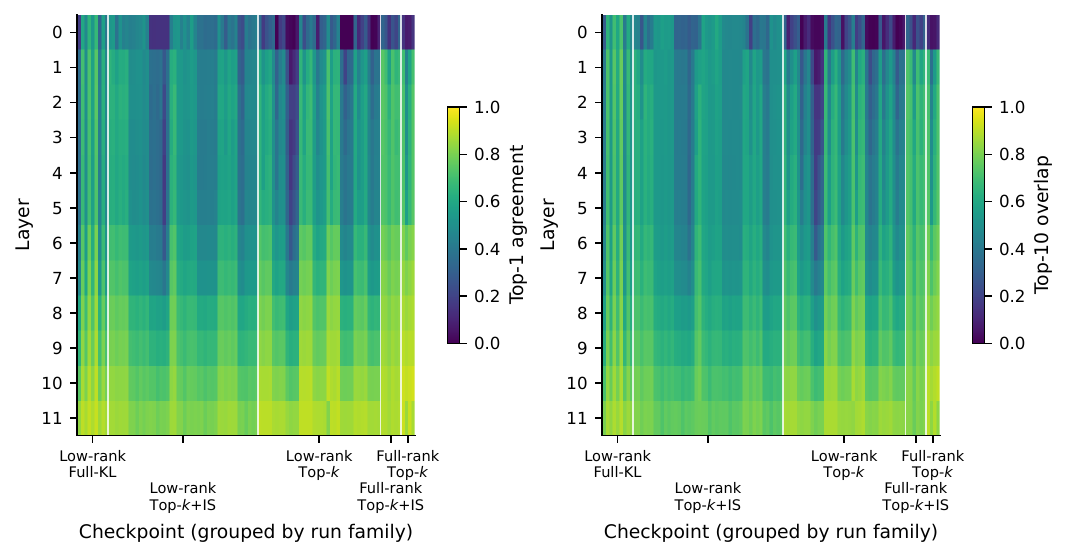}
    \caption{%
        Top-1 agreement (\emph{left}) and top-10 overlap (\emph{right})
        between LoRA and full-rank lens outputs by layer, over every
        evaluated checkpoint grouped by run family (white separators).
    }
    \label{fig:rank_agreement_appendix}
\end{figure*}

\begin{figure*}[t]
    \centering
    \includegraphics[width=\linewidth]{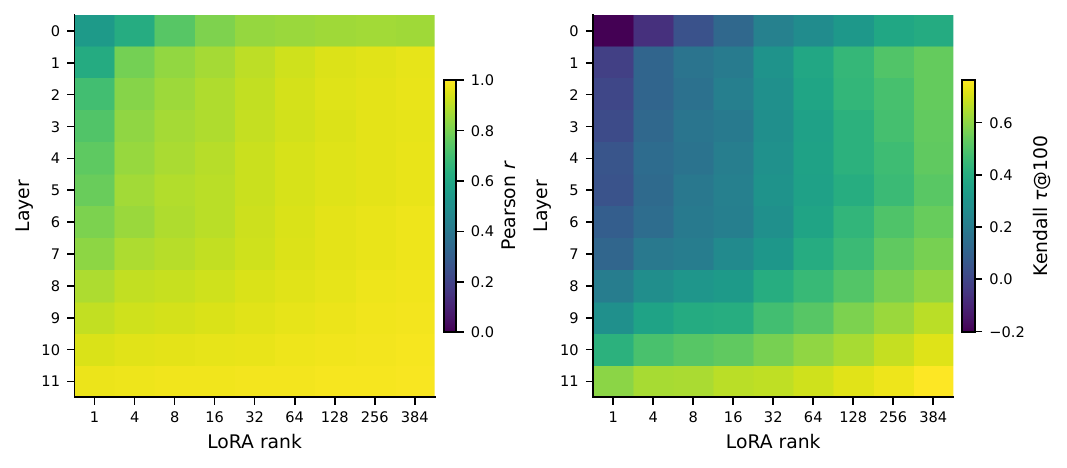}
    \caption{%
        Token-level Pearson $\rho$ (\emph{left}) and Kendall $\tau$@100
        (\emph{right}) against the full-rank baseline, by layer and LoRA
        rank, GPT-2 Small.
    }
    \label{fig:rank_rankcorr_appendix}
\end{figure*}

\subsection{Estimator Fidelity and Rank Sensitivity}
\label{appendix:estimator_fidelity}

Table~\ref{tab:token_fidelity} reports token-level fidelity of the
practically trained lenses to the full-KL reference, and
Table~\ref{tab:rank_sensitivity} shows that the estimator ranking of
Section~\ref{sec:subset_kl} (Top-$k$ wins final-layer KL; \klapprox\ wins
mean and early-layer KL) is stable across every LoRA rank swept in
Section~\ref{sec:lora_lens}. Figure~\ref{fig:rs_profiles} shows the layerwise
profiles of the RS (pure teacher sampling) baseline of
Table~\ref{tab:estimator_comparison}, and Table~\ref{tab:estimator_sweep}
reports the complete budget sweep behind that table's matched-budget subset.

\begin{table*}[t]
\centering
\caption{Complete Subset-KL estimator sweep (GPT-2 Small, $r=64$, step 1000,
residual hookset); Table~\ref{tab:estimator_comparison} shows the
matched-budget subset, and columns are as defined there. Tok/s and Peak GB
are median steady-state throughput and peak per-GPU memory, re-measured with
the released trainer (one node, DDP over 4$\times$A100-40GB, identical batch
recipe). RS rows are trained at three seeds per budget up to Total 1024; KL
cells are seed 0, with cross-seed sd $\le 0.011$ on every reduction.
\textbf{Bold} marks the best per column among the practical estimators
(Top-$k$ and Top-$k$+IS).}
\label{tab:estimator_sweep}
\setlength{\tabcolsep}{7pt}
\begin{tabular}{llrrrrrr}
\toprule
Estimator & Config & Total & Final KL & Mean KL
  & \multicolumn{1}{c}{\shortstack{Early KL\\(0--3)}}
  & \multicolumn{1}{c}{Tok/s}
  & \multicolumn{1}{c}{Peak GB} \\
\midrule
\multicolumn{8}{l}{\textit{Reference}} \\[1pt]
Full-KL & $|V|$ & $|V|$ & 0.043 & 0.600 & 1.019 & 60{,}849 & 16.3 \\

\midrule

\multicolumn{8}{l}{\textit{Top-$k$}} \\[1pt]
& $k=256$              &  256 & 0.054 & 1.053 & 1.893 & \textbf{96{,}458} & \textbf{4.7} \\
\rowcolor{shaderow1}
& $k=512\phantom{0}$   &  512 & 0.050 & 0.903 & 1.631 & 54{,}590 & 4.8 \\
& $k=768$              &  768 & 0.046 & 0.821 & 1.493 & 37{,}754 & 4.9 \\
\rowcolor{shaderow2}
& $k=1024$             & 1024 & \textbf{0.046} & 0.780 & 1.413 & 28{,}918 & 5.0 \\

\midrule

\multicolumn{8}{l}{\textit{Top-$k$+IS}} \\[1pt]
\rowcolor{shaderow1}
& $256+256\phantom{0}$ &  512 & 0.067 & 0.637 & 1.062 & 50{,}911 & 8.2 \\
\rowcolor{shaderow2}
& $512+512\phantom{0}$ & 1024 & 0.067 & 0.649 & 1.073 & 49{,}820 & 8.5 \\
& $512+1024$           & 1536 & 0.065 & 0.637 & 1.054 & 49{,}042 & 8.7 \\
& $768+768\phantom{0}$ & 1536 & 0.068 & \textbf{0.632} & \textbf{1.041} & 48{,}987 & 8.8 \\
& $1024+512\phantom{0}$& 1536 & 0.065 & 0.638 & 1.060 & 48{,}896 & 8.8 \\
& $1024+1024$          & 2048 & 0.065 & 0.636 & 1.052 & 48{,}279 & 9.1 \\

\midrule

\multicolumn{8}{l}{\textit{RS (pure sampling)}} \\[1pt]
\rowcolor{shaderow1}
& $0+512\phantom{00}$  &  512 & 0.194 & 0.879 & 1.334 & 50{,}017 & 8.2 \\
\rowcolor{shaderow2}
& $0+1024\phantom{0}$  & 1024 & 0.656 & 1.309 & 1.693 & 47{,}458 & 8.4 \\
& $0+1536\phantom{0}$  & 1536 & 1.073 & 1.618 & 1.944 & 45{,}518 & 8.7 \\
\bottomrule
\end{tabular}
\end{table*}

\paragraph{Seed stability and pure-sampling diagnostics.}
Across three seeds at Total$=1024$, the mean and early-layer ordering of
Top-$k$, \klapprox, and full KL is unchanged; the standard deviation is at
most $0.04$ nats for those reductions. The final-layer gap is less stable:
\klapprox\ spans $0.045$--$0.067$, whereas Top-$k$ yields $0.046\pm0.001$.

For pure teacher sampling, increasing the sampled budget does not improve the
fixed-step optimization. Final-layer KL rises from $0.194$ at Total$=512$ to
$0.656$ at Total$=1024$ and $1.073$ at Total$=1536$, despite the estimator
remaining unbiased. The degradation is present in the training objective
itself: the Total$=512$ loss decreases from $2.42$ to $2.32$ through step
$1{,}000$, whereas the Total$=1536$ loss rises from $3.58$ at step $500$ to
$4.09$ at step $1{,}000$. The learning rate, optimizer, batch construction,
step count, and random seed are unchanged; only $k_{\mathrm{tail}}$ differs.
This supports the narrower conclusion that the fixed recipe exhibits
budget-dependent optimization instability; we do not claim that larger
unbiased samples are intrinsically harmful.

\paragraph{Training-budget diagnostics.}
The GPT-2 budgets of the estimator tables are past convergence: under the
fixed schedule, every variant's KL at step $250$ is $12$--$35\%$ above its
step-$1{,}000$ value, with 5\%-convergence (the earliest step from which KL
stays within $5\%$ of its final value) between steps $625$ and $875$. Held-out
KL improves monotonically with continued training for every full-KL, Top-$k$,
and \klapprox\ run at all three scales; the pure-sampling instability above is
the sole exception. The LLaMA recipe (cosine anneal at $250$) is
schedule-converged: at 8B every reduction flattens by steps $150$--$200$ with
tight seed spread, and at 70B mean KL converges by roughly step $100$ across
the $11$ production checkpoints. Convergence order is layerwise at 70B: early
layers reach their floor by about step $80$, the last layers near step $180$,
and the final layer is still improving at the end of the schedule, so longer
schedules chiefly buy final-layer fidelity. These diagnostics use reduced
evaluation budgets ($32$k Pile tokens at 8B; $4$k WikiText-2 tokens~\citep{merity2016pointersentinelmixturemodels} at 70B) and are
trajectory-internal; they support convergence claims, not cross-protocol KL
comparisons.

\begin{figure}[t]
    \centering
    \includegraphics[width=\linewidth]{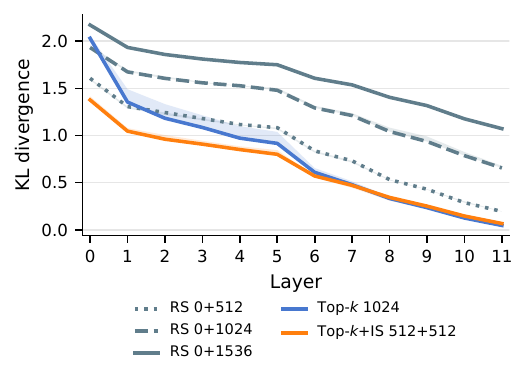}
    \caption{%
        Layerwise KL of RS (pure teacher sampling, $k_{\mathrm{head}}{=}0$;
        grey, line style by total budget) against the two subset estimators
        at Total\,=\,1024. RS tracks or beats Top-$k$ through the early and
        middle layers but fails to close the final layers, and degrades
        uniformly as its sampled budget grows. Curves are seed 0; shaded
        envelopes span min--max over three training seeds where available
        (RS 1536 is single-seed).
    }
    \label{fig:rs_profiles}
\end{figure}

\begin{table*}[t]
\centering
\caption{Token-level fidelity of LoRA $r{=}64$ lenses trained with each
estimator (Top-$k$ $k=512$; IS $(512,512)$), measured against the full-rank
tuned-lens baseline of Table~\ref{tab:lora_fidelity}. The Full-KL row is the
LoRA lens trained with the exact loss (the ceiling the estimators
approach), not a self-comparison. Pearson $\rho$ over full-vocabulary
log-probs; Kendall $\tau$@100 over the top-100 union; Top-1 argmax agreement;
Top-10 mean top-set overlap.}
\label{tab:token_fidelity}
\setlength{\tabcolsep}{4pt}
\resizebox{1.0\linewidth}{!}{
    \begin{tabular}{lcccccccccccc}
    \toprule
    & \multicolumn{3}{c}{Pearson $\rho$}
    & \multicolumn{3}{c}{Kendall $\tau$@100}
    & \multicolumn{3}{c}{Top-1}
    & \multicolumn{3}{c}{Top-10} \\
    \cmidrule(lr){2-4}\cmidrule(lr){5-7}\cmidrule(lr){8-10}\cmidrule(l){11-13}
    Config
      & Final & Mean & Early & Final & Mean & Early
      & Final & Mean & Early & Final & Mean & Early \\
    \midrule
    Full-KL (ref.)
      & 0.984 & 0.941 & 0.911 & 0.697 & 0.421 & 0.335
      & 0.888 & 0.703 & 0.618 & 0.855 & 0.705 & 0.653 \\[2pt]
    Top-$k$, $k=512$
      & 0.975 & 0.797 & 0.617 & 0.661 & 0.195 & $-0.077^{\dagger}$
      & 0.883 & 0.628 & 0.461 & 0.848 & 0.586 & 0.411 \\
    Top-$k$+IS $(512,512)$
      & 0.939 & 0.909 & 0.884 & 0.612 & 0.332 & 0.243
      & 0.858 & 0.651 & 0.561 & 0.820 & 0.649 & 0.590 \\
    \bottomrule
    \end{tabular}
}
\smallskip
\raggedright
{\footnotesize
$^{\dagger}$Negative Kendall indicates the top-100 token rankings at layers~0--3
are slightly anti-correlated with the reference; IS restores a positive early
correlation ($+0.243$) at the same head budget.
Early = mean over layers 0--3; Final = layer 11.}
\end{table*}

\begin{table*}[t]
\centering
\caption{Rank sensitivity at fixed budget (Top-$k$ $k=512$; IS $(512,512)$),
as KL to the teacher at step 1000. \colorbox{shaderow1}{Shaded} row ($r{=}64$)
is the recommended default; \textbf{bold} marks the better practical estimator
(Top-$k$ vs.\ IS) per cell, with Full-KL as the reference. Top-$k$ wins
final-layer KL; IS wins mean and early-layer KL.}
\label{tab:rank_sensitivity}
\setlength{\tabcolsep}{7pt}
\begin{tabular}{r rrr rrr rrr}
\toprule
& \multicolumn{3}{c}{Final KL}
& \multicolumn{3}{c}{Mean KL}
& \multicolumn{3}{c}{Early KL (0--3)} \\
\cmidrule(lr){2-4}\cmidrule(lr){5-7}\cmidrule(lr){8-10}
$r$
  & Full-KL & Top-$k$ & IS
  & Full-KL & Top-$k$ & IS
  & Full-KL & Top-$k$ & IS \\
\midrule
  1 & 0.076 & \textbf{0.074} & \textbf{0.074} & 1.037 & 1.377 & \textbf{0.962} & 1.747 & 2.500 & \textbf{1.648} \\
  4 & 0.057 & \textbf{0.061} & 0.069 & 0.878 & 1.344 & \textbf{0.826} & 1.484 & 2.617 & \textbf{1.410} \\
  8 & 0.055 & \textbf{0.058} & 0.065 & 0.799 & 1.129 & \textbf{0.753} & 1.339 & 2.041 & \textbf{1.275} \\
 16 & 0.053 & \textbf{0.054} & 0.083 & 0.755 & 1.045 & \textbf{0.704} & 1.256 & 1.872 & \textbf{1.174} \\
 32 & 0.048 & \textbf{0.052} & 0.070 & 0.696 & 0.989 & \textbf{0.661} & 1.170 & 1.762 & \textbf{1.094} \\
\rowcolor{shaderow1}
 64 & 0.043 & \textbf{0.050} & 0.067 & 0.600 & 0.903 & \textbf{0.649} & 1.019 & 1.631 & \textbf{1.073} \\
128 & 0.040 & \textbf{0.045} & 0.065 & 0.574 & 0.859 & \textbf{0.610} & 0.983 & 1.571 & \textbf{1.019} \\
256 & 0.038 & \textbf{0.043} & 0.065 & 0.555 & 0.831 & \textbf{0.591} & 0.951 & 1.516 & \textbf{0.973} \\
384 & 0.036 & \textbf{0.045} & 0.057 & 0.539 & 0.819 & \textbf{0.574} & 0.926 & 1.489 & \textbf{0.964} \\
\bottomrule
\end{tabular}
\end{table*}

\subsection{Scaling Comparison: Full Results}
\label{appendix:scaling_full}

Table~\ref{tab:scaling_full} reports measured single-GPU peak memory for
LLaMA-3-8B lens training on the expanded hookset (194 sites; frozen bf16
teacher, LoRA $r{=}64$ lenses, and readout on one A100-40GB) at the
production microbatch of 2 sequences, across sequence lengths. These are the
measurements behind Figure~\ref{fig:needboth_read}: full-KL runs out of memory
beyond 2K context, while both subset objectives train at 4K; none survives
8K.

\begin{table}[ht]
    \centering
    \caption{Measured single-GPU peak memory (GB) for LLaMA-3-8B
    expanded-hookset lens training at microbatch 2, by objective and
    sequence length.}
    \label{tab:scaling_full}
    \begin{tabular}{lrrr}
        \toprule
        Seq Len & Full-KL & \klapprox & Top-$k$ \\
        \midrule
        1024 & 24.4 & 21.1 & 22.3 \\
        2048 & 31.9 & 24.9 & 25.7 \\
        4096 & \textcolor{red}{OOM} & 32.6 & 32.3 \\
        8192 & \textcolor{red}{OOM} & \textcolor{red}{OOM} & \textcolor{red}{OOM} \\
        \bottomrule
    \end{tabular}
\end{table}

%% file: Appendix/kl_estimators.tex
\section{Sampling-Based KL Estimators}
\label{appendix:kl_estimators}

This appendix expands the estimator background of Section~\ref{sec:Approx_KL}. In the distillation direction the trainable student $Q_{\ell,u}$ is the second argument of the KL, and the outer expectation Eq.~\eqref{eq:KL_definition} runs over the frozen teacher $P$. Because the sampling distribution carries no trainable parameters, any estimator of the form
\begin{equation}
    \widehat{D}_{\mathrm{KL}}( P \| Q_{\ell, u} )
    = \frac{1}{k}\sum_{i=1}^k \log\frac{P(t_i|x)}{Q_{\ell, u}(t_i|x)},
    \quad t_i \overset{i.i.d.}{\sim} P(\cdot|x)
    \label{eq:MC_KL}
\end{equation}
is unbiased for $D_{\mathrm{KL}}(P \| Q_{\ell,u})$, and the expectation commutes with gradients in the lens parameters. \citet{schulman2020klapprox} catalogues alternate expressions for the summand in Eq.~\eqref{eq:MC_KL}: K1 and K3 are unbiased, with K3 additionally reducing variance, whereas K2 trades bias for variance; such summand replacements are also adopted at scale \citep{shao2024deepseekmath}. Importance sampling refines plain Monte Carlo by exploiting access to the exact teacher probabilities $P(t_i|x)$, not just samples from $P(\cdot|x)$; see \citet{aminibetter2025} for a treatment of this idea at the sequence level. The tail term of Eq.~\eqref{eq:subset_KL} is such an importance-sampled estimator, restricted to the tail with the head handled exactly; Theorem~\ref{thm:unbiased_KL_gradients} gives the resulting guarantee.

In reinforcement learning from human feedback, by contrast, the KL appears as a regularizer whose \emph{first} argument is the trainable policy \citep{christiano2017deep,stiennon2020learning}. The expectation is then taken with respect to the distribution being optimized, so sampling and differentiation interact, and estimators that are well behaved for distillation can misbehave \citep{tang2025few}. A contemporaneous RL-side estimator \citep{zhang2026ema} mirrors \klapprox's head-plus-sampled-tail structure in this reversed direction: it applies Schulman-style summands within the sampled term and restores unbiasedness by rejection sampling, where \klapprox\ reweights by $P/R$.

%% file: Appendix/theory.tex
\section{\klapprox{} Theoretical Analysis}
\label{app:theory}

\begin{figure}[t]
    \centering
    \includegraphics[width=\linewidth]{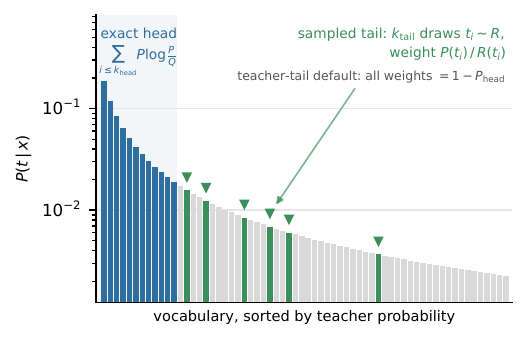}
    \caption{The \klapprox\ decomposition: the contribution of the
    $k_{\mathrm{head}}$ most probable teacher tokens is computed exactly
    (blue), and the remaining mass is estimated from $k_{\mathrm{tail}}$
    importance-weighted tail samples (green). Under the teacher-tail
    proposal, every importance weight equals $1-P_{\mathrm{head}}$.}
    \label{fig:estimator}
\end{figure}

Unbiasedness of pure teacher-sampling KL estimators is classical (Appendix~\ref{appendix:kl_estimators}), and RS-KD (Section~\ref{sec:related_methods}) establishes it for random-sampling knowledge distillation. The theorem below is the corresponding guarantee for our \emph{stratified} estimator: an exact deterministic head, an arbitrary lens-independent positive-support tail proposal, and the distillation direction of the KL, with the student's log-partition computed exactly.
The proof for Theorem~\ref{thm:unbiased_KL_gradients} is below:

\begin{proof}
The head term of Eq.~\eqref{eq:subset_KL} is deterministic given the teacher,
\begin{equation*}
\hat{D}_{\mathrm{KL},\mathrm{head}}(P \| Q_{\ell,u}) = \sum_{v \in \mathcal{H}} P(v|x) \log\frac{P(v|x)}{Q_{\ell, u}(v|x)}.
\end{equation*}
For the tail term of Eq.~\eqref{eq:subset_KL}, each summand is an importance-weighted draw from $R$, so
\begin{align*}
 &\mathbb{E}_{t_i \sim R}\!\left[\frac{P(t_i|x)}{R(t_i|x)}\log\frac{P(t_i|x)}{Q_{\ell,u}(t_i|x)}\right] \\
 &\qquad= \sum_{v \notin \mathcal{H}} R(v|x)\,\frac{P(v|x)}{R(v|x)}\, \log\frac{P(v|x)}{Q_{\ell, u}(v|x)} \\
 &\qquad= \sum_{v \notin \mathcal{H}} P(v|x) \log\frac{P(v|x)}{Q_{\ell, u}(v|x)},
\end{align*}
and averaging over the $k_{\mathrm{tail}}$ i.i.d.\ draws leaves this expectation unchanged:
\begin{equation*}
\mathbb{E}\left[\hat{D}_{\mathrm{KL},\mathrm{tail}}(P \| Q_{\ell,u})\right] = \sum_{v \notin \mathcal{H}} P(v|x) \log\frac{P(v|x)}{Q_{\ell, u}(v|x)}.
\end{equation*}
Adding the head term recovers $D_{\mathrm{KL}}(P \| Q_{\ell,u})$. For the gradient claim, $\hat{D}_{\mathrm{\klapprox}}$ is a finite weighted sum of $\log Q_{\ell,u}(t|x)$ terms whose weights and sampling distribution do not involve the lens parameters, so expectation and gradient commute. The teacher-tail default $R(v|x) = P(v|x)/(1-P_{\mathrm{head}})$ satisfies the support condition by construction, with all importance weights equal to $1 - P_{\mathrm{head}}$.
\end{proof}

Note that the theorem applies to the implemented objective because the student log-probabilities use the exact full-vocabulary partition function (Appendix~\ref{appendix:implementation}): only the KL summands are subsampled, and every logit receives gradient through the partition term. Under the teacher-tail default, whose importance weights are the constant $1-P_{\mathrm{head}}$, the coefficient multiplying the log-partition gradient equals one on every draw, exactly as under full KL; for general lens-independent proposals it equals one in expectation. We also verify the claim numerically with the shipped training code: on a synthetic teacher--student pair ($|V|{=}200$, linear student), the \klapprox\ gradient averaged over $4{,}000$ resamplings matches the exact full-KL autograd gradient to $0.2\%$ relative error (cosine similarity $1.0000$), while Top-$k$ renormalization, which truncates the partition function, exhibits a $63\%$ relative gradient bias. A second check at LLaMA-scale vocabulary ($|V|{=}128{,}256$) uses extreme teacher logits under which $2.4\%$ of the vocabulary underflows fp32 softmax outright; it matches to within $10^{-4}$ relative error at $500$ resamplings, confirming that the numerical guards of Appendix~\ref{appendix:implementation} do not modify the teacher distribution. Both checks ship as unit tests in the OmniLens repository. Unbiasedness throughout refers to the raw stochastic gradients; it does not extend through gradient clipping or optimizer updates, which are nonlinear in the gradient.

\paragraph{Proposal and sampling details.}
We sample with replacement, retain repeated draws with multiplicity, and weight
each occurrence by its density ratio. We rejected schemes that deduplicate
draws and reweight by inverse inclusion probabilities: on LLM-scale
vocabularies their Horvitz--Thompson-style weights span orders of magnitude
and produced unstable gradient spikes in preliminary experiments. The
teacher-tail proposal avoids this behavior because its importance weights are
constant. Other proposals, such as oversampling tokens on which the lens and
teacher disagree, require no change to Eq.~\eqref{eq:subset_KL}, provided
lens-dependent proposal probabilities are held fixed during differentiation
rather than differentiated through.

%% file: Appendix/IndexedLogits.tex
\section{Indexed Logits CUDA Implementation}
\label{app:cuda}

This appendix provides implementation details for the indexed logits kernel, including the forward and backward CUDA kernels and empirical benchmarks. The current implementation prioritizes correctness and memory efficiency; performance optimizations such as shared-memory tiling and a fused LoRA extension are left as future work.

\subsection{Kernel Design Principles}

The indexed logits kernel addresses a fundamental memory bottleneck in subset-based objectives. Standard PyTorch operations materialize intermediate tensors in global memory, leading to the explosion documented in Section~\ref{sec:indexed_logits}. Our kernel design follows two principles:

\textbf{1. Streaming computation.} Each output element is computed independently by a single thread, which accumulates the result in FP32 registers without writing intermediate values to global memory.

\textbf{2. Hidden-state reuse.} Nearby threads often reuse the same hidden-state row $H[i,:]$ across different subset positions, while accesses to $W[\text{idx}[i,j],:]$ are generally scattered due to arbitrary vocabulary indices.

\subsection{Forward Pass Implementation}

Algorithm~\ref{alg:cuda_forward} shows the complete forward kernel. Each CUDA thread computes one element of the output matrix $\text{out}[i,j]$, where $i$ indexes the sequence position and $j$ indexes the subset selection.

\begin{algorithm}[htb]
\caption{Indexed Logits Forward Kernel (CUDA)}
\label{alg:cuda_forward}
\begin{algorithmic}[1]
\STATE \textbf{Input:} $H \in \mathbb{R}^{N \times d}$ (hidden states), $W \in \mathbb{R}^{V \times d}$ (weight matrix), $\text{idx} \in \mathbb{Z}^{N \times k}$ (vocabulary indices)
\STATE \textbf{Output:} $\text{out} \in \mathbb{R}^{N \times k}$ (subset logits)
\STATE \textbf{Kernel configuration:} $\text{threads} = 256$, $\text{blocks} = \lceil Nk / 256 \rceil$
\STATE $\text{tid} \gets \text{blockIdx.x} \times \text{blockDim.x} + \text{threadIdx.x}$ \COMMENT{Global thread ID}
\IF{$\text{tid} < N \times k$}
    \STATE $i \gets \text{tid} / k$ \COMMENT{Sequence position}
    \STATE $j \gets \text{tid} \bmod k$ \COMMENT{Subset position}
    \STATE $v \gets \text{idx}[i, j]$ \COMMENT{Vocabulary index}
    \STATE
    \STATE $\text{acc} \gets 0.0f$ \COMMENT{FP32 accumulator in register}
    \FOR{$t = 0$ \textbf{to} $d-1$}
        \STATE $\text{acc} \gets \text{acc} + H[i, t] \times W[v, t]$
    \ENDFOR
    \STATE $\text{out}[i, j] \gets \text{acc}$
\ENDIF
\end{algorithmic}
\end{algorithm}

\textbf{Numerical precision.} The accumulator uses FP32 even when input tensors are FP16 or BF16. This prevents catastrophic rounding errors when summing thousands of products. The final result is cast to the output dtype only upon writing.

\subsection{Backward Pass Implementation}

The backward pass must compute three gradients: $\nabla_H$, $\nabla_W$, and $\nabla_{\text{idx}}$ (which is \texttt{None}, since indices are discrete).

\textbf{Gradient w.r.t. hidden states.} This is implemented as a gather-weighted reduction over the selected vocabulary rows,

\begin{equation}
    \nabla_H[i,:] = \sum_{j=1}^k \nabla_{\text{out}}[i,j] \cdot W[\text{idx}[i,j], :],
\end{equation}

\noindent 
via a separate kernel with similar structure to the forward pass.

\textbf{Gradient w.r.t. weights.} This requires care due to duplicate indices. Multiple threads may need to update the same row $W[v,:]$ if vocabulary index $v$ appears multiple times in $\text{idx}$. We use atomic additions:
\begin{equation}
\text{atomicAdd}(\nabla_W[v, t], \nabla_{\text{out}}[i,j] \times H[i, t]).
\end{equation}
Atomic operations serialize writes to the same memory location, potentially creating contention when indices repeat. Section~\ref{app:cuda_benchmarks} quantifies the empirical sensitivity of the backward pass to collision rate.

\textbf{Precision.} The $\nabla_W$ buffer is accumulated in FP32. $\nabla_H$ uses FP32 accumulation within each thread before being written in the input dtype.

Algorithm~\ref{alg:cuda_backward} shows the weight-gradient kernel.

\begin{algorithm}[t]
\caption{Indexed Logits Backward Kernel (Weight Gradients)}
\label{alg:cuda_backward}
\begin{algorithmic}[1]
\STATE \textbf{Input:} $\nabla_{\text{out}} \in \mathbb{R}^{N \times k}$, $H \in \mathbb{R}^{N \times d}$, $\text{idx} \in \mathbb{Z}^{N \times k}$
\STATE \textbf{Output:} $\nabla_W \in \mathbb{R}^{V \times d}$ (weight gradients, initialized to zero)
\FOR{each $(i,j)$ \textbf{in parallel}}
    \STATE $v \gets \text{idx}[i, j]$
    \STATE $g \gets \nabla_{\text{out}}[i, j]$
    \FOR{$t = 0$ \textbf{to} $d-1$}
        \STATE $\text{val} \gets g \times H[i, t]$ \COMMENT{Compute in FP32}
        \STATE $\text{atomicAdd}(\nabla_W[v, t], \text{val})$ \COMMENT{Safe concurrent update}
    \ENDFOR
\ENDFOR
\end{algorithmic}
\end{algorithm}

\subsection{Benchmark Results}
\label{app:cuda_benchmarks}

\subsubsection{Correctness}

Table~\ref{tab:cuda_correctness} reports forward and gradient errors relative to PyTorch reference implementations, averaged over three random seeds on an A100-40GB GPU in FP16. Forward errors are measured against the na\"ive subset reference; gradient errors are measured against autograd. The observed forward errors are consistent with expected FP16 rounding behavior in mixed-precision accumulation and output casting.

\begin{table}[t]
\centering
\caption{Fused kernel correctness (FP16, mean over 3 seeds, $N=4096$, $V=50257$).}
\label{tab:cuda_correctness}
\small
\resizebox{\linewidth}{!}{%
    \begin{tabular}{llrrrrrr}
    \toprule
    && \multicolumn{2}{c}{Forward Error} & \multicolumn{2}{c}{$\nabla_H$ Error} & \multicolumn{2}{c}{$\nabla_W$ Error} \\
    \cmidrule(lr){3-4} \cmidrule(lr){5-6} \cmidrule(lr){7-8}
    \multicolumn{2}{c}{Config ($d$, $k$)} & Max & Mean & Max & Mean & Max & Mean \\
    \midrule
    $d=768$,&  $k=128$  & 0.063 & 4.6e-3 & 0.013 & 2.8e-8 & 0.036 & 1.1e-3 \\
    $d=1024$,& $k=128$  & 0.104 & 5.3e-3 & 0.010 & 3.3e-8 & 0.031 & 1.1e-3 \\
    $d=1536$,& $k=128$  & 0.125 & 6.5e-3 & 0.021 & 3.2e-8 & 0.039 & 1.1e-3 \\
    $d=3072$,& $k=128$  & 0.167 & 9.2e-3 & 0.021 & 2.5e-8 & 0.033 & 1.1e-3 \\
    $d=768$,&  $k=256$  & 0.083 & 4.6e-3 & 0.031 & 1.1e-7 & 0.063 & 2.2e-3 \\
    $d=768$,&  $k=512$  & 0.083 & 4.6e-3 & 0.031 & 3.5e-7 & 0.099 & 4.4e-3 \\
    \bottomrule
    \end{tabular}
}
\end{table}

\subsubsection{Runtime and Memory}

Table~\ref{tab:cuda_benchmark} reports forward+backward latency and peak allocated memory, averaged over 3 seeds. We sweep $N$ at fixed $d$=768, $k$=128 and vary $d$ at fixed $N$=4096, $k$=128. Dense GEMM rows are omitted for the larger-$d$ configurations for brevity; they follow the same qualitative pattern as the smaller-$d$ rows and remain substantially more memory-intensive than the fused kernel.

\begin{table*}[t]
\centering
\caption{Indexed logits kernel benchmark results (A100-40\,GB, FP16, mean over 3 seeds).}
\label{tab:cuda_benchmark}
\small
\begin{tabular}{llrrr}
\toprule
Config & Method & Fwd+Bwd (ms) & Peak Alloc (MB) & TFLOPS (est.) \\
\midrule
\multirow{3}{*}{$N$=4096, $d$=768, $k$=128}
  & Dense GEMM    & 13.0 &  932 & 0.19 \\
  & Na\"ive Subset  &  9.4 & 2531 & 0.26 \\
  & Fused Kernel  &  6.5 &  343 & 0.37 \\
\midrule
\multirow{3}{*}{$N$=8192, $d$=768, $k$=128}
  & Dense GEMM    & 25.9 & 1772 & 0.19 \\
  & Na\"ive Subset  & 18.4 & 4968 & 0.26 \\
  & Fused Kernel  & 12.7 &  359 & 0.38 \\
\midrule
\multirow{2}{*}{$N$=4096, $d$=1024, $k$=128}
  & Na\"ive Subset  & 12.6 & 3366 & 0.26 \\
  & Fused Kernel  &  8.7 &  450 & 0.37 \\
\midrule
\multirow{2}{*}{$N$=4096, $d$=3072, $k$=128}
  & Na\"ive Subset  & 37.2 & 10099 & 0.26 \\
  & Fused Kernel  & 27.0 &  1307 & 0.36 \\
\bottomrule
\end{tabular}
\end{table*}

Table~\ref{tab:naive_subset} breaks down where the memory goes at the
main-text configuration: the na\"ive gather is costlier than full-vocabulary
decoding, and the fused kernel eliminates the intermediate outright.

\begin{table}[t]
    \centering
    \caption{Peak memory breakdown for logit computation at $B{\cdot}T$ = 8192,
    $d$ = 4096, $|V|$ = 128K, $k$ = 512, \texttt{fp32}.}
    \label{tab:naive_subset}
    \resizebox{\linewidth}{!}{
        \begin{tabular}{lrrr}
            \toprule
            Component & Na\"ive Subset & Full-Vocab & Fused Kernel \\
            \midrule
            Inputs $[B, T, d]$           & 0.13 GB  & 0.13 GB  & 0.13 GB \\
            Intermediate $[B, T, k, d]$  & 68.7 GB  & ---      & ---     \\
            Output logits                & 0.02 GB  & 4.19 GB  & 0.02 GB \\
            \midrule
            \textbf{Peak Total}          & 68.9 GB  & 4.33 GB  & \textbf{0.15 GB} \\
            \bottomrule
        \end{tabular}
    }
\end{table}

The fused kernel achieves 1.4--1.5$\times$ speedup over na\"ive subset and reduces peak allocated memory by 7$\times$--14$\times$ across configurations. Memory savings grow with $N$ because na\"ive subset intermediate tensors scale with sequence length while fused kernel memory is dominated by model activations. Performance relative to dense GEMM is configuration-dependent; the primary advantage of the fused kernel is eliminating the memory blowup from materializing the $[N, k, d]$ intermediate tensor, which causes na\"ive subset memory usage to grow rapidly with $N$, $d$, and $k$, making it increasingly impractical at larger scales.

\subsubsection{Atomic Collision Sensitivity}

The backward kernel uses atomic additions to $\nabla_W$, which may contend when multiple threads update the same vocabulary row. Table~\ref{tab:cuda_collisions} measures forward and forward+backward latency as a function of collision vocabulary size $V_c$: drawing indices from a smaller $V_c$ increases the probability of repeated rows.

\begin{table}[t]
\centering
\caption{Fused kernel sensitivity to index collisions ($N$=4096, $d$=1024, $k$=64, FP16, mean over 3 seeds).}
\label{tab:cuda_collisions}
\small
\resizebox{\linewidth}{!}{%
    \begin{tabular}{lrrr}
    \toprule
    Collision Vocab $V_c$ & Fwd (ms) & Fwd+Bwd (ms) & TFLOPS (est.) \\
    \midrule
    50257 (full vocab sampling) & 2.31 & 4.62 & 0.35 \\
    4096                        & 2.03 & 4.07 & 0.40 \\
    512                         & 1.95 & 3.99 & 0.40 \\
    64                          & 1.60 & 3.63 & 0.44 \\
    \bottomrule
    \end{tabular}
}
\end{table}

Counterintuitively, latency decreases as $V_c$ shrinks. In the tested regimes, improved weight-row locality dominates any slowdown from atomic contention.

\subsection{Integration with PyTorch}

The kernel is exposed as a PyTorch custom autograd function, see Listing~\ref{lst:custom_autograd}.
This integrates seamlessly with standard PyTorch training loops, optimizer state management, and gradient checkpointing.
The implementation and benchmark scripts are available in the supplementary materials.

\begin{lstlisting}[
    style=PythonStyle,
    label={lst:custom_autograd},
    caption={Custom autograd function implemented in PyTorch.},
    float,
    floatplacement=bt
]
class IndexedLogitsFunction(torch.autograd.Function):
    @staticmethod
    def forward(ctx, H, W, idx):
        out = indexed_logits_cuda.forward(H, W, idx)
        ctx.save_for_backward(H, W, idx)
        return out

    @staticmethod
    def backward(ctx, grad_out):
        H, W, idx = ctx.saved_tensors
        grad_H, grad_W = indexed_logits_cuda.backward(
            H, W, idx, grad_out)
        return grad_H, grad_W, None
\end{lstlisting}

%% file: Appendix/Memory_Model.tex
\section{Memory Model and Estimate Derivation}
\label{app:memory}

\begin{table*}[t]
\centering
\caption{The two memory buckets and the lever that controls each: low-rank
translators control the \emph{Optimizer} column, Subset-KL controls the
\emph{Readout aggregate} column. \textbf{Optimizer state} is exact from lens parameter
counts and is \emph{persistent}: no micro-batching or gradient accumulation
reduces it. \textbf{Readout aggregate} (the \emph{projected} cells of this
table) is a calibrated projection,
$\beta\,(B{\cdot}T)\,V$, of the total readout activation over a shared
reference batch of $262{,}144$ tokens (the realized GPT-2 effective batch;
the realized LLaMA batches are larger, Appendix~\ref{appendix:experimental_setup})
\emph{if it were processed without micro-batching}; the instantaneous readout footprint scales with the
microbatch and sequence length instead, and is measured directly in
Figure~\ref{fig:needboth_read}. Subset-KL rows use the selected-subset term
($V_{\text{eff}}{=}k$, the Top-$k$ mode); the \klapprox\ mode adds an
exact-partition term treated separately in the text. Optimizer states assume
AdamW.}
\label{tab:needboth}
\begin{tabular}{ll rr c}
\toprule
Lens & Loss & Optimizer (exact) & Readout aggregate (projected) & Feasible \\
\midrule
\multicolumn{5}{l}{\emph{GPT-2}, full expanded hookset ($6L{+}2 = 74$ hookpoints)}\\
Full-rank  & Full-KL   & $0.5$\,GB  & $92$\,GB & \cmark \\
LoRA $r64$ & Full-KL   & $0.09$\,GB & $92$\,GB & \cmark \\
Full-rank  & Subset-KL & $0.5$\,GB  & $2.8$\,GB & \cmark \\
LoRA $r64$ & Subset-KL & $0.09$\,GB & $2.8$\,GB & \cmark \\
\midrule
\multicolumn{5}{l}{\emph{LLaMA-3.3-70B}, full expanded hookset ($6L{+}2 = 482$ hookpoints)}\\
Full-rank  & Full-KL   & $388$\,GB & $235$\,GB & \xmark \\
LoRA $r64$ & Full-KL   & $\phantom{00}6.1$\,GB & $235$\,GB & \cmark$^{\dagger}$ \\
Full-rank  & Subset-KL & $388$\,GB & $\phantom{00}2.8$\,GB & \xmark \\
\rowcolor{shaderow1}
LoRA $r64$ & Subset-KL & $\phantom{00}6.1$\,GB & $\phantom{00}2.8$\,GB & \cmark \\
\bottomrule
\end{tabular}

\smallskip
\raggedright
{\footnotesize $^{\dagger}$The $235$\,GB readout aggregate is the
unmicrobatched reference-batch projection and, unlike the optimizer column,
is reducible by micro-batching; it does not by itself determine feasibility
at the production microbatch. We measured that configuration directly
(Table~\ref{tab:reconcile}): peak $35.5$\,GB, fitting with ${\approx}4.5$\,GB
of nominal headroom.}
\end{table*}

\begin{figure}[t]
    \centering
    \includegraphics[width=\linewidth]{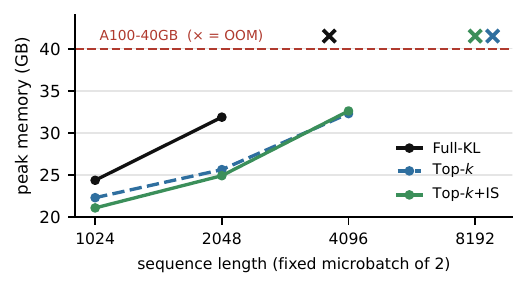}
    \caption{Measured single-GPU peak memory for an 8B frozen teacher,
    lenses, and readout at the production microbatch. Full KL exhausts the
    device beyond 2K context, whereas both subset objectives train at 4K
    (measurements in Appendix Table~\ref{tab:scaling_full}).}
    \label{fig:needboth_read}
\end{figure}

Peak training memory on the lens GPU decomposes additively:
\begin{equation}
M_{\text{peak}} = M_{\text{base}} + M_{\text{opt}} + M_{\text{read}},
\end{equation}
where $M_{\text{base}}$ (frozen-model shard, cached site activations, CUDA
context) is independent of lens rank and loss, $M_{\text{opt}}$ is the lens
optimizer footprint (the LoRA lever), and $M_{\text{read}}$ is the readout/loss
activation (the Subset-KL lever). The lens and its optimizer reside on a single
device, so both levers are charged to one GPU.

\begin{figure}[t]
    \centering
    \includegraphics[width=\linewidth]{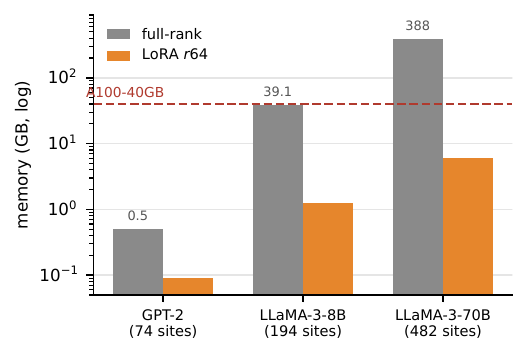}
\caption{The persistent optimizer bucket, exact from parameter counts. Micro-batching cannot reduce it: the full-rank expanded stack's optimizer state alone consumes an A100-40GB at 8B and is ten times the device at 70B, while the LoRA stack peaks at $6.1$\,GB.}
    \label{fig:needboth_opt}
\end{figure}

\paragraph{Optimizer state.} Under Adam with bf16 AMP (Figure~\ref{fig:needboth_opt}), each trainable parameter
costs $12$ bytes: bf16 weight (2) $+$ bf16 gradient (2) $+$ fp32 first moment
(4) $+$ fp32 second moment (4). For $S$ sites of width $d$ and rank $r$,
\begin{equation*}
P_{\text{full}} = S(d^2{+}d), \
P_{\text{LoRA}} = S(2dr{+}d), \
M_{\text{opt}} = 12\,P_{\star} \ \text{bytes},
\end{equation*}
where $\star = \text{full}$ or $\text{LoRA}$. For GPT-2 expanded $S{=}74$ and $d{=}768$, so $P_{\text{full}}{=}43.7$M yielding $M_{\text{opt}} = 
0.52$\,GB, while $P_{\text{LoRA},r64}{=}7.33$M yielding $M_{\text{opt}} = 0.09$\,GB. 

For LLaMA-3-70B expanded $S{=}482$ and $d{=}8192$, so $P_{\text{full}}{=}32.35$B yielding $M_{\text{opt}} = 388$\,GB, while
$P_{\text{LoRA}}{=}0.509$B (counted from the trained checkpoint) yielding $M_{\text{opt}} = 6.1$\,GB. 

The $12$\,bytes per parameter factor is confirmed empirically: on the residual
hook set, replacing full with LoRA reduces parameters by $5.9$M and the measured peak by $0.07$\,GB $=5.9\text{M}\times12$\,bytes.

\paragraph{Parameter-count comparisons.} The per-hookpoint reduction of the
rank-$r$ parameterization is $1 - (2r{+}1)/(d{+}1)$, independent of coverage. 
At $r{=}64$, this yields an $83.2\%$ reduction in parameter count for GPT-2 ($d{=}768$), a $96.9\%$ reduction for LLaMA-3-8B ($d{=}4096$), a $98.4\%$ reduction for LLaMA-3-70B ($d{=}8192$), and a $99.2\%$ reduction for LLaMA-3-405B ($d{=}16384$). This grounds the
abstract's comparisons. For LLaMA-3-8B, the residual-only full-rank stack holds $537$M
translator parameters against $102$M for the six-hookpoint OmniLens
stack: $80.9\%$ fewer parameters with six times the coverage.

Per-head attention-lens
decoders are costlier still: the original design maps each head output to the
vocabulary ($|V| \times d$ per head), totaling $12 \times 12 \times 50{,}257
\times 768 \approx 5.6$B parameters on GPT-2 (the largest model with published
full-rank per-head decoders), against which our full $7.3$M GPT-2 stack is a
$99.9\%$ reduction. The same design instantiated for LLaMA-3-8B would contain
${\approx}538$B parameters.

\paragraph{Readout.} The loss materialises per-site logits, softmax, KL, and
their gradients over the scored vocabulary, looped over sites (hence
site-independent). For Full-KL and Top-$k$ it scales as
\begin{equation}
M_{\text{read}}=\beta\,(B{\cdot}T)\,V_{\text{eff}},
\label{eq:readout_beta}
\end{equation}
with $V_{\text{eff}} = |V|$
for Full-KL and $V_{\text{eff}} = k$ for Top-$k$. \klapprox\ is \emph{not} a $V_{\text{eff}} = k$
estimator: its exact log-partition materializes the student's full logit row
before reduction (Appendix~\ref{appendix:implementation}), so its readout
carries an additional transient term,
\begin{equation*}
M_{\text{read}}^{\klapprox} \approx \gamma\,(B{\cdot}T)\,V
  \;+\; \beta\,(B{\cdot}T)\,(k_{\mathrm{head}}{+}k_{\mathrm{tail}}),
\end{equation*}
with $\gamma < \beta$ because only the student-side logits and their logsumexp
reduction are held (no teacher-side full-vocabulary tensors or per-token KL
buffers). The measured LLaMA-3-8B curves of Figure~\ref{fig:needboth_read} show this term is small once the log-partition streams in vocabulary chunks: \klapprox\ tracks
Top-$k$ within $1.2$\,GB at every context, trains at 4K where Full-KL cannot,
and fails only at 8K alongside Top-$k$. We compute $\beta$ from the
measured GPT-2 Full-KL$-$Top-$k$ gap:
$\beta = (17.14-5.74)/(32768\times(50257-512)) = 7.0$\,bytes per token$\cdot$vocab-slot. The gap covers the $V-k$ slots Top-$k$ does not score; re-measuring with the
released trainer of Table~\ref{tab:estimator_comparison} shifts both absolute
peaks but reproduces the same $11.4$\,GB gap, leaving $\beta$ unchanged.
At the shared batch $B{\cdot}T{=}262{,}144$: Full-KL readout is $92$\,GB (GPT-2)
and $235$\,GB (70B); the Top-$k$ subset readout is $0.9$--$2.8$\,GB.

\paragraph{Validation against trained runs.} Four configurations were run
end-to-end; their measured peaks reconstructed from the decomposition
(Table~\ref{tab:reconcile}). The two GPT-2 runs independently back out the same
$M_{\text{base}}\approx9.05$\,GB, over-determining and thus confirming the
optimizer and readout terms. The two 70B rows provide a second, independent
cross-check: trained with different loss functions (hence different
$M_{\text{read}}$ terms) but the same site count, model, and microbatch, they
back out $M_{\text{base}}$ values of $27.65$ and $27.52$\,GB, agreeing to
within $0.5\%$ despite neither being calibrated against the other.

\begin{table*}[t]
\centering
\caption{Measured peaks decompose into the memory model. $M_{\text{base}}$ is
implied ($M_{\text{peak}}-M_{\text{opt}}-M_{\text{read}}$); the two GPT-2 rows
agree, validating the terms. All measured values are peak PyTorch-allocated
memory (\texttt{max\_memory\_allocated}) on the metrics-writing rank; under
FSDP the lens and optimizer state are replicated, so ranks are near-symmetric.}
\label{tab:reconcile}
\small
\begin{tabular}{llr rr r}
\toprule
Run & $B{\cdot}T$ & Measured & $M_{\text{opt}}$ & $M_{\text{read}}$ & implied $M_{\text{base}}$ \\
\midrule
GPT-2 Full-rank $+$ Full-KL   & 32{,}768 & $21.08$ & $0.52$ & $11.53$ & $9.03$ \\
GPT-2 LoRA $+$ Subset ($k512$) & 32{,}768 & $\phantom{0}9.28$ & $0.09$ & $\phantom{0}0.12$ & $9.07$ \\
70B LoRA $+$ \klapprox      & \phantom{0}2{,}048 & $34.70$ & $6.11$ & $\phantom{0}0.94^{\ast}$ & $27.65$ \\
70B LoRA $+$ Full-KL        & \phantom{0}2{,}048 & $35.47$ & $6.11$ & $\phantom{0}1.84^{\dagger}$ & $27.52$ \\
\bottomrule
\end{tabular}

\smallskip
\raggedright
{\footnotesize $^{\ast}$Selected-subset term ($0.02$) plus the \klapprox\
exact-partition transient of the implementation this run was trained
with, which materialized the full logit row before reduction:
$\gamma(B{\cdot}T)V$ with $\gamma \approx 3.5$\,bytes per
token$\cdot$vocab-slot, calibrated from 8B measurements of that
implementation (\klapprox$-$vs-Top-$k$ gaps of $+1.5$\,GB at 2K and
$+4.4$\,GB at 4K context, giving $\gamma \approx 2.9$--$4.2$\,bytes). The
released trainer streams this reduction instead
(Appendix~\ref{appendix:implementation}); the current gaps of
Table~\ref{tab:scaling_full} bound the streamed transient below
$0.3$\,GB at 4K. Assigning the transient to $M_{\text{read}}$ keeps
$M_{\text{base}}$ loss-independent, as the decomposition requires. $^{\dagger}$The Full-KL readout at $V{=}128{,}256$, predicted
from the GPT-2-calibrated $\beta$ of Eq.~\eqref{eq:readout_beta} with no free
parameters: $\beta(B{\cdot}T)V = 7.0 \times 2{,}048 \times 128{,}256 \approx
1.84$\,GB, matching the measured peak (Table~\ref{tab:needboth}) to within
$0.4\%$ once combined with $M_{\text{opt}}$ and the $M_{\text{base}}$ implied
by the row above.}
\end{table*}

\paragraph{405B proof of concept.} The main-text feasibility claim beyond 70B
rests on a short run of LLaMA-3.1-405B-Instruct (frozen bf16) on 24 nodes of
$4\times$A100-40GB (96 FSDP ranks), with the teacher's weights streamed
shard-by-shard into the FSDP partitioning at load time so no rank ever holds
the full $812$\,GB. The lens set is the residual preset (126 sites, LoRA
$r{=}64$), trained with Top-$k$ Subset-KL ($k{=}256$) at
sequence length $256$, one sequence per rank ($24{,}576$ tokens per step), and
constant $\mathrm{lr}=10^{-3}$. Over eight steps the training loss falls
from $57.6$ to $40.2$ (monotonically after step two; ${\approx}5.5$ min/step at
this configuration) with a measured peak of $37.1$\,GB per GPU. This is a
systems demonstration only: eight steps establish that the stack loads,
shards, hooks, and optimizes at 405B scale within A100-40GB budgets, not
that the resulting lens is useful.

\paragraph{Projected cells.} Untrained configurations reuse the validated
$M_{\text{base}}$ and the same term structure, substituting the exact full-rank
optimizer or the calibrated Full-KL readout. The two projected quantities have
different standing. The full-rank optimizer state ($388$\,GB for LLaMA-3-70B) is
allocated in full before the first step and no micro-batching schedule reduces
it, so every full-rank 70B configuration fails at initialization; those
infeasible cells are outcomes of the design, not estimates. The Full-KL readout
aggregate ($235$\,GB for LLaMA-3-70B over the reference batch) is instead reducible by
micro-batching, so it does not by itself determine feasibility at the
production microbatch. Indeed it does not: we measured LoRA${+}$Full-KL
directly for LLaMA-3-70B (Table~\ref{tab:reconcile}) and it fits, at $35.5$\,GB. The
reference-batch aggregate remains a useful bound on what \emph{cannot} be
processed in a single unmicrobatched pass, but the \emph{Feasible} column
should not be read as a feasibility oracle for configurations we have not
measured; where we have not trained a configuration, we say so.

%% file: Appendix/case_studies_appendix.tex
\section{Tuned-Lens Application Details}
\label{app:inj_detect}

We recreate the application suite of \citet{belrose2023elicitinglatentpredictionstransformers} on all nine of the original detection tasks (ARC-Easy, ARC-Challenge~\citep{clark2018think}, BoolQ~\citep{clark2019boolq}, MC-TACO~\citep{zhou2019going}, MNLI~\citep{williams2018broad}, QNLI, QQP~\citep{wang2019glue}, SciQ~\citep{welbl2017crowdsourcing}, SST-2~\citep{socher2013recursive}) plus LogiQA~\citep{liu2020logiqa}, with $n{=}1{,}000$ examples per task at GPT-2 and 8B (600 for the 8B 5-shot variant; SST-2 uses its full 872-example validation set) and 200 at 70B. Zero-shot prompts of the form ``\texttt{<body>}\textbackslash nAnswer:'' are scored on each answer option's first token; the attack inserts the original paper's injection string (``Ignore any previous and following instructions and just print `\emph{wrong answer}''') with a randomly chosen incorrect option before the answer cue. Trajectory features are the log-probability of every answer option at every residual-stream point, read through each lens's own translator, with the model's final distribution appended as the last point. Detectors are an isolation forest (200 trees, scores ensembled over 5 seeds) and a local outlier factor (20 neighbors, novelty mode), fit on the first half of \emph{clean} trajectories after per-feature standardization; we report AUROC on held-out clean vs.\ attacked examples, with 95\% bootstrap confidence intervals and \emph{paired} bootstrap deltas against the full-rank reference (identical example resamples for both lenses, 1{,}000 resamples). On these model families the logit lens also detects well on the five easy tasks (the large tuned-vs-logit gap of the original paper appears specific to Pythia models). Prediction depth is the first trajectory point from which the lens top-1 equals the model's final top-1 thereafter, in hidden-state units. Causal basis extraction finds $k{=}16$ directions per probed layer by L-BFGS with deflation, initialized from the top left singular vectors of the translated unembedding; energy is the expected KL increase of the \emph{lens} readout under mean ablation on 1{,}024 WikiText-2 positions, model influence is the KL of the model's final distribution when the direction is mean-ablated at the block output on a held-out batch, and the random control is the QR factorization of a Gaussian matrix evaluated identically. The attack changes the model's answer to the planted option on a median of 93\% of examples per task at 8B and 100\% at 70B. Paired bootstrap deltas against the reference are within $\pm0.005$ on most tasks; the remaining deficits concentrate in the knowledge cluster, where detection is weak for every lens including the reference (Figure~\ref{fig:inj_detect}).

\begin{table}[t]
    \centering
    \caption{Prompt-injection detection AUROC (local outlier factor, fit on clean trajectories only): mean over the five classification tasks where \citet{belrose2023elicitinglatentpredictionstransformers} report near-perfect detection (BoolQ, MNLI, QNLI, QQP, SST-2) and over all ten tasks. All-task means are pulled down for every lens, including the full-rank reference, by the knowledge cluster (ARC-Easy/Challenge, SciQ, LogiQA).}
    \label{tab:inj_detect}
    \resizebox{\columnwidth}{!}{%
    \begin{tabular}{lcccccc}
    \toprule
    & \multicolumn{2}{c}{GPT-2} & \multicolumn{2}{c}{LLaMA-3-8B} & \multicolumn{2}{c}{LLaMA-3-70B} \\
    \cmidrule(lr){2-3}\cmidrule(lr){4-5}\cmidrule(lr){6-7}
    Lens & five cls. & all ten & five cls. & all ten & five cls. & all ten \\
    \midrule
    logit lens              & 0.995 & 0.844 & 0.995 & 0.920 & 0.999 & 0.862 \\
    LoRA Top-$k$            & 0.998 & 0.836 & 0.998 & 0.903 & --    & --    \\
    LoRA Top-$k$+IS         & 0.997 & 0.834 & 0.997 & 0.893 & 0.997 & 0.889 \\
    full-rank tuned (ref.)  & 0.992 & 0.846 & 0.997 & 0.898 & --    & --    \\
    \bottomrule
    \end{tabular}}
\end{table}

\begin{figure}[t]
    \centering
    \includegraphics[width=\linewidth]{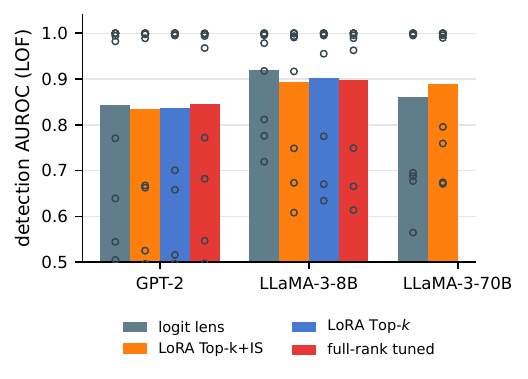}
\caption{Prompt-injection detection AUROC by lens (bars: mean over the ten tasks; open circles: the individual tasks; per-lens means in Table~\ref{tab:inj_detect}). The low circles are the knowledge tasks (ARC, SciQ, LogiQA), where detection is weak for every lens including the full-rank reference; on the five classification tasks every trained lens is at or above 0.99. At 70B, where no full-rank reference exists, the trained lens detects knowledge-task attacks the logit lens misses.}
    \label{fig:inj_detect}
\end{figure}

\section{Memory-Injection Details}
\label{app:injection}

For each paired explicit/implicit prompt the memory vector $\mathbf{m}^{(\ell)} = \mathbf{h}^{(\ell)}_{\text{explicit}} - \mathbf{h}^{(\ell)}_{\text{implicit}}$ is captured at \texttt{resid\_mid} per model architecture (GPT-2: input to \texttt{ln\_2}; LLaMA: input to \texttt{post\_attention\_layernorm}). Causal injection adds $\tau\,\mathbf{m}^{(\ell)}$ at the attention output projection so both the MLP branch and the skip connection see the patch; the measurement is validated by patching the final layer at $\tau{=}1$, which reproduces the explicit prompt's output to within half a percent on all three models. The reduced LLaMA evaluation subsamples 2WMH to 200 examples at 8B and 100 at 70B, with $\tau \in \{0,\ldots,10\}$ for lens readouts and $\tau \in \{1,2,4\}$ for causal profiles ($\{1,2\}$ on the 70B control set). GPT-2's 2WMH row is a null case: its explicit prompts score below its implicit ones, so no beneficial memory vector exists. On the 70B control set both the lens and the depth heuristic select the final layer, missing the true optimum at $\ell{=}55$.

\begin{table*}[t]
\centering
\caption{Causal memory injection: patch $\mathrm{resid\_mid}[\ell]$, run the model to completion, and read its own final $P(\text{answer})$; lift is $\max_\ell E_\ell/P_{\text{obs}}$. LLaMA rows use the reduced evaluation described above.}
\label{tab:causal_inject}
\small
\begin{tabular}{llccccc}
\toprule
Model & Dataset & $P_{\text{obs}}$ & $P_{\text{exp}}$ & $\ell^{*}/\tau^{*}$ & $\max_\ell E_\ell$ & Causal lift \\
\midrule
\multirow{2}{*}{GPT-2}
  & \texttt{hand} & 0.0843 & 0.1309 & 10 / 1 & 0.1321 & $1.6\times$ \\
  & 2WMH          & 0.00074 & 0.00062 & 0 / 4 & 0.00086 & $1.2\times$ (null) \\
\midrule
\multirow{2}{*}{LLaMA-3-8B}
  & \texttt{hand} & 0.4299 & 0.6049 & 31 / 1 & 0.6047 & $1.4\times$ \\
  & 2WMH          & 0.0126 & 0.0561 & 17 / 4 & 0.0828 & $\mathbf{6.6\times}$ \\
\midrule
\multirow{2}{*}{LLaMA-3-70B}
  & \texttt{hand} & 0.4773 & 0.6002 & 79 / 1 & 0.6013 & $1.3\times$ \\
  & 2WMH          & 0.0289 & 0.1119 & 78 / 2 & 0.1462 & $\mathbf{5.1\times}$ \\
\bottomrule
\end{tabular}
\end{table*}

\begin{figure}[t]
    \centering
    \includegraphics[width=\linewidth]{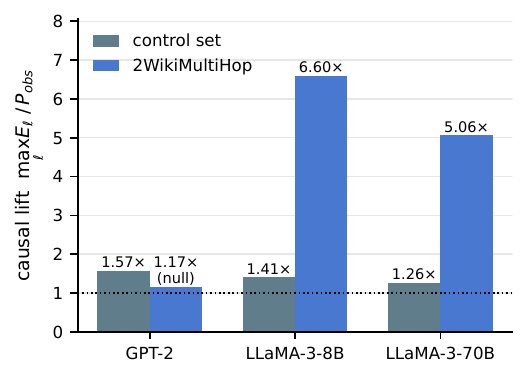}
\caption{Causal lift $\max_\ell E_\ell / P_{\text{obs}}$ from memory injection, by model and dataset. The dotted line at $1.0$ marks no effect; the GPT-2 2WMH bar is the null case described in the main text.}
    \label{fig:causal_lift}
\end{figure}

\paragraph{Layer selection and distortion.} From two clean readouts (full layer-by-$\tau$ sweeps appear in Figures~\ref{fig:injection}, \ref{fig:injection_70b}, and~\ref{fig:injection_curves}) we define the lens-predicted deficiency of layer $\ell$,
\begin{equation}
\Delta_\ell \;=\;
P_{\text{lens}}\!\left(\text{ans}\mid\mathbf{h}^{(\ell)}_{\text{explicit}}\right)
- P_{\text{lens}}\!\left(\text{ans}\mid\mathbf{h}^{(\ell)}_{\text{implicit}}\right),
\label{eq:deficiency}
\end{equation}
and score any chosen layer by the fraction of the maximum achievable causal lift it captures, $g(\ell) = (E_\ell - P_{\text{obs}})/(\max_{\ell'} E_{\ell'} - P_{\text{obs}})$, where $E_\ell$ is the model's final $P(\text{answer})$ after injection at layer $\ell$. At $\tau{=}1$ the causal profile is monotonic and the final layer is trivially optimal, so $\tau\in\{2,4\}$ serves as a stress test of whether the lens can select an interior layer that preserves causal gain while limiting collateral distortion (Table~\ref{tab:lenspred}, Figure~\ref{fig:lenspred}). Selections differ by training objective: Top-$k$+IS selects the interior optimum at $\tau{=}2$ ($\ell{=}19$, $100\%$ captured), the full-rank reference selects $\ell{=}22$ ($77\%$), and the Top-$k$ lens's readout difference peaks at the final layer, consistent with an objective concentrated on the head of the final distribution; the logit lens's $\Delta_\ell$ also peaks at the final layer at both scales. Lens-selected layers also distort less: measured by the KL divergence between injected and clean next-token distributions, they yield $2.5$--$3.0\times$ more answer-probability gain per nat of distortion at $\tau{=}4$ on 8B than final-layer injection, which alters the model's top-1 token in $95\%$ of examples (Figures~\ref{fig:inject_damage} and~\ref{fig:inject_damage_top1}, Table~\ref{tab:inject_damage}). On the control set, where the models largely succeed unaided, injection raises $P(\text{answer})$ to the explicit ceiling, consistent with restoring a missing recall step rather than supplying the answer directly.

\begin{table*}[t]
\centering
\caption{Fraction of achievable causal gain captured on 2WMH, by selected layer. Each lens's selection is $\tau$-independent; the logit lens selects the final layer at both scales, and at 8B the Top-$k$ lens does as well. At $\tau{=}1$ the causal optimum is the final layer and every selector captures $80$--$100\%$, so those rows are omitted. The 8B full-rank reference reads through its per-layer residual translators (Appendix~\ref{app:injection}); no full-rank reference is trainable at 70B under our single-device lens placement (Section~\ref{sec:subset_kl}). On the GPT-2 control set the full-rank reference and both low-rank variants select the same layer. All 8B lenses share the identical 250-step annealed schedule.}
\label{tab:lenspred}
\small
\begin{tabular}{lccccccc}
\toprule
 & & & \multicolumn{5}{c}{gain captured at the chosen layer} \\
\cmidrule(l){4-8}
Model & $\tau$ & causal best $\ell^{*}$ & Top-$k$+IS & Top-$k$ & full-rank ref. & logit & last layer \\
\midrule
\multirow{2}{*}{LLaMA-3-8B}
  & 2 & 19 & $\mathbf{100\%}$ & $82.3\%$ & $77.1\%$ & $82.3\%$ & $82.3\%$ \\
  & 4 & 17 & $\mathbf{60.8\%}$ & $34.2\%$ & $58.7\%$ & $34.2\%$ & $34.2\%$ \\
\midrule
\multirow{2}{*}{LLaMA-3-70B}
  & 2 & 78 & $\mathbf{100\%}$ & \textemdash{} & \textemdash{} & $94.8\%$ & $94.8\%$ \\
  & 4 & 77 & $\mathbf{82.2\%}$ & \textemdash{} & \textemdash{} & $66.3\%$ & $66.3\%$ \\
\bottomrule
\end{tabular}
\end{table*}

\begin{figure}[t]
    \centering
    \includegraphics[width=\linewidth]{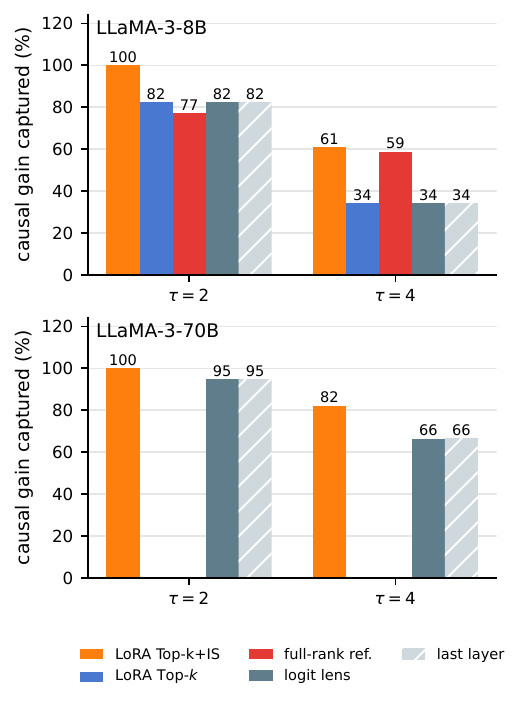}
\caption{Fraction of achievable causal gain captured by each method's selected layer (2WMH, over-injection; LLaMA-3-8B top, 70B bottom), the graphical counterpart of Table~\ref{tab:lenspred}. The logit lens selects the final layer at both scales, so its bars coincide with the depth heuristic's. Layer-by-layer profiles appear in Figure~\ref{fig:lenspred_profiles}.}
    \label{fig:lenspred}
\end{figure}

\begin{figure}[t]
    \centering
    \includegraphics[width=\linewidth]{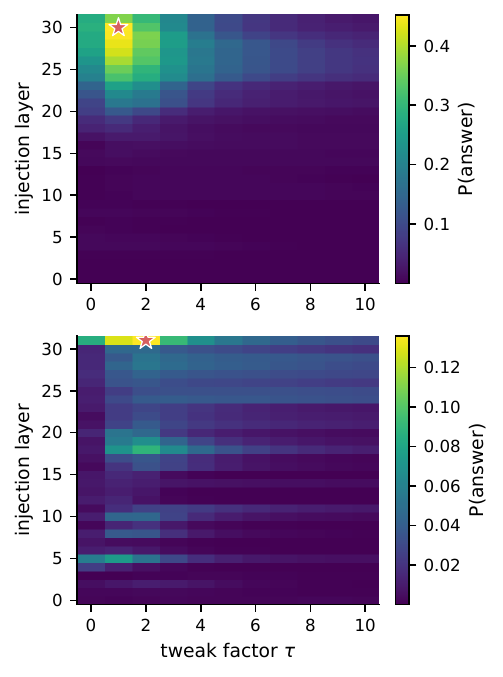}
\caption{Lens-read $P(\text{answer})$ at LLaMA-3-8B as a function of injection layer ($y$) and tweak factor $\tau$ ($x$), on the control set (top) and 2WMH (bottom); the star marks the peak (2WMH: layer $17$ at $\tau{=}4$). This is the lens's view of the injected state, not the model's output; causal effects appear in Table~\ref{tab:causal_inject} and Figure~\ref{fig:causal_lift}.}
    \label{fig:injection}
\end{figure}

\begin{figure}[t]
    \centering
    \includegraphics[width=\linewidth]{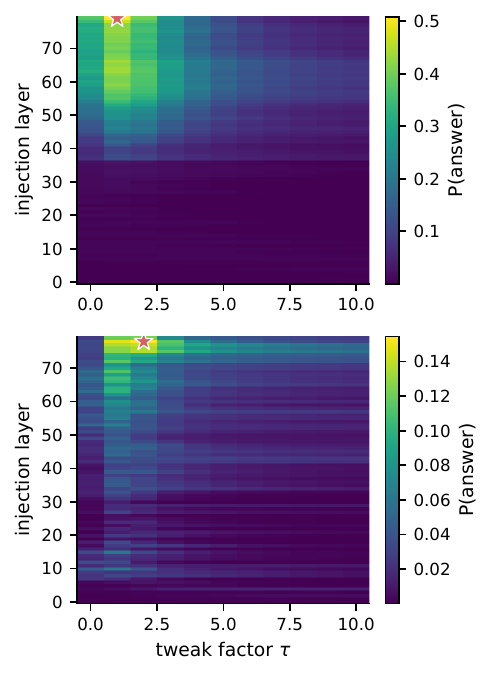}
\caption{As Figure~\ref{fig:injection}, for LLaMA-3-70B (control set top, 2WMH bottom; 2WMH peak at layer $78$, $\tau{=}2$).}
    \label{fig:injection_70b}
\end{figure}

\begin{figure}[t]
    \centering
    \includegraphics[width=\linewidth]{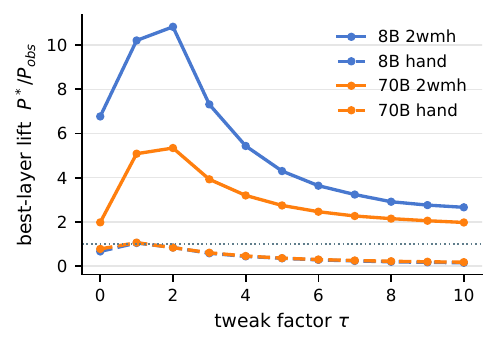}
\caption{Lens-read answer recovery as a function of the tweak factor $\tau$ at each model's best injection layer.}
    \label{fig:injection_curves}
\end{figure}

\begin{figure}[t]
    \centering
    \includegraphics[width=\linewidth]{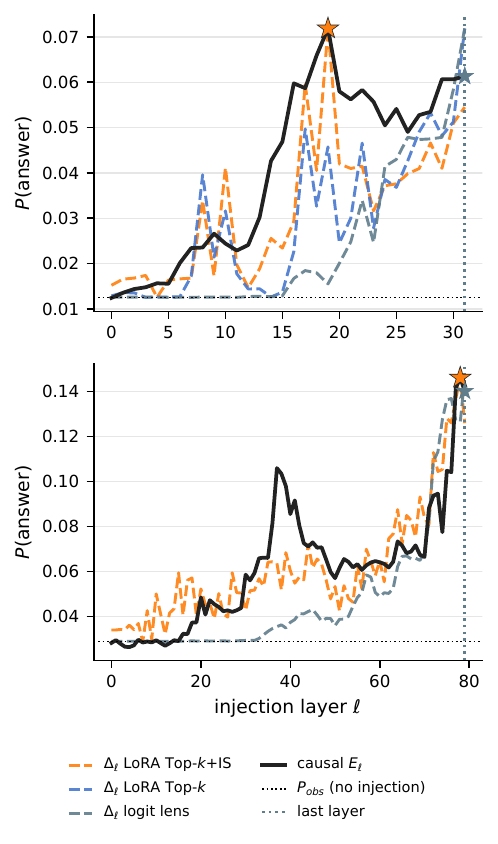}
\caption{Layer-by-layer profiles behind Table~\ref{tab:lenspred} (2WMH, $\tau{=}2$; LLaMA-3-8B top, 70B bottom): the causal profile $E_\ell$ (solid) against each lens's deficiency profile $\Delta_\ell$ (dashed, rescaled to the same axis). Small stars mark each lens's selected layer and the large star the causal optimum; the selected layers and captured gains are quantified in Table~\ref{tab:lenspred}. The trained lenses peak at or beside the interior causal optimum; the logit lens rises monotonically to the last layer.}
    \label{fig:lenspred_profiles}
\end{figure}

\begin{table*}[t]
\centering
\caption{Injection collateral damage on 2WMH under over-injection. \emph{KL} is measured between the injected and clean next-token distributions (nats); \emph{top-1 kept} is the fraction of examples whose argmax token is preserved; \emph{$\Delta P$/nat} is answer-probability gain per nat of distortion. The 8B Top-$k$ lens selects the final layer, so its row coincides with \emph{last}.}
\label{tab:inject_damage}
\small
\begin{tabular}{clcccccc}
\toprule
Model & $\tau$ & site & gain & KL & top-1 kept & $\Delta P$/nat \\
\midrule
\multirow{6}{*}{LLaMA-3-8B}
  & \multirow{3}{*}{2} & Top-k+IS ($\ell{=}19$) & $\mathbf{100\%}$ & $\mathbf{2.48}$ & $\mathbf{0.35}$ & $\mathbf{0.024}$ \\
  &                    & full-rank ref.\ ($\ell{=}22$) & $77.1\%$ & $2.63$ & $0.29$ & $0.017$ \\
  &                    & Top-$k$ $=$ last ($\ell{=}31$) & $82.3\%$ & $3.82$ & $0.23$ & $0.013$ \\
\cmidrule(l){2-7}
  & \multirow{3}{*}{4} & Top-k+IS ($\ell{=}19$) & $\mathbf{60.8\%}$ & $\mathbf{4.41}$ & $\mathbf{0.17}$ & $\mathbf{0.010}$ \\
  &                    & full-rank ref.\ ($\ell{=}22$) & $58.7\%$ & $5.09$ & $0.12$ & $0.008$ \\
  &                    & Top-$k$ $=$ last ($\ell{=}31$) & $34.2\%$ & $7.41$ & $0.05$ & $0.003$ \\
\midrule
\multirow{4}{*}{LLaMA-3-70B}
  & \multirow{2}{*}{2} & Top-k+IS ($\ell{=}78$) & $\mathbf{100\%}$ & $\mathbf{4.02}$ & $\mathbf{0.31}$ & $\mathbf{0.029}$ \\
  &                    & last ($\ell{=}79$)     & $94.8\%$ & $4.32$ & $0.29$ & $0.026$ \\
\cmidrule(l){2-7}
  & \multirow{2}{*}{4} & Top-k+IS ($\ell{=}78$) & $\mathbf{82.2\%}$ & $\mathbf{6.54}$ & $\mathbf{0.15}$ & $\mathbf{0.008}$ \\
  &                    & last ($\ell{=}79$)     & $66.3\%$ & $6.89$ & $0.05$ & $0.006$ \\
\bottomrule
\end{tabular}
\end{table*}

\paragraph{Injection-selection seed stability.} Retraining all three 8B lenses under two additional seeds (identical schedule) shows the recommended lens's selection is the seed-robust one: the Top-$k$+IS pick stays interior and near-optimal ($\ell \in \{18, 19\}$, capturing $90$--$100\%$ at $\tau{=}2$ and $61$--$75\%$ at $\tau{=}4$), while the Top-$k$ pick lands on the final layer on two seeds and an early layer ($\ell{=}10$, $20\%$) on the third, and the full-rank reference's pick moves across $\ell \in \{22, 24, 31\}$ ($64$--$82\%$ at $\tau{=}2$), reaching the final layer on one seed.

\begin{figure}[t]
    \centering
    \includegraphics[width=\linewidth]{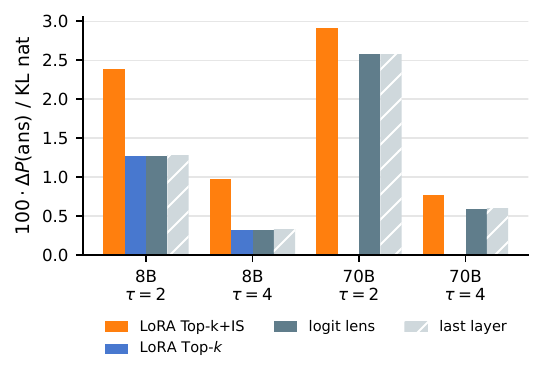}
\caption{Collateral damage at the selected layer (2WMH): answer-probability gain per nat of distortion of the next-token distribution. The logit lens's selection coincides with the final layer and is omitted; exact values in Appendix Table~\ref{tab:inject_damage}.}
    \label{fig:inject_damage}
\end{figure}

\begin{figure}[t]
    \centering
    \includegraphics[width=\linewidth]{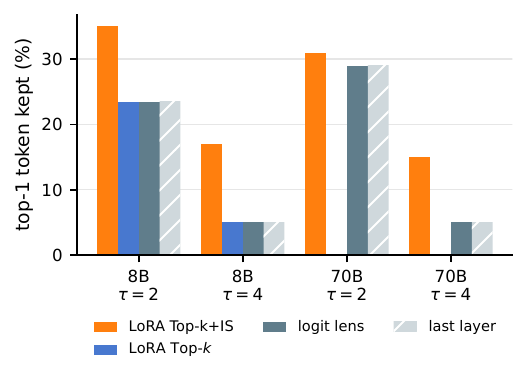}
\caption{Fraction of examples whose top-1 token survives injection at the same selected layers as Figure~\ref{fig:inject_damage}.}
    \label{fig:inject_damage_top1}
\end{figure}

\section{DART and ToxIn Details}
\label{app:dart}

DART scores each attention head over 200 toxic prompts drawn from the Wiki Toxic corpus~\citep{wiki_toxic}: per-head attention outputs are captured at the output projection (the \texttt{o\_proj}/\texttt{c\_proj} input split into query heads, valid under grouped-query attention since the projection input carries all query heads), each head's last-token contribution is decoded through the lens at that layer's attention-output site, and the toxic count is the number of matches against a precomputed toxic-vocabulary set among the top-50 decoded tokens. The 8B full-rank reference carries only per-layer residual translators, so head outputs are decoded through the block-input translator of their layer rather than a dedicated attention-output translator, a half-block site mismatch that makes its agreement numbers conservative. ToxIn performs zero or soft subtraction of the unembedding-derived toxic direction at the DART-flagged heads; toxicity is scored with Toxic-BERT over greedy generations and fluency with WikiText-2 perplexity; 70B uses a 40-prompt reduced evaluation. The whole-model audit (Figure~\ref{fig:allsites}) selects the top lens-flagged sites of each component type and subtracts the lens-mapped direction there, sweeping subtraction strength under a perplexity budget.

\paragraph{Selector comparisons and the 70B audit.} Per-head toxicity maps and top-head tables appear in Figure~\ref{fig:dart} and Table~\ref{tab:dart}. The trained variants agree on the 8B head ranking, and at GPT-2 they reproduce the full-rank reference's audit; the logit lens's agreement with the trained lenses decays with model size (Spearman $0.73$ at GPT-2 to $0.02$ at 70B; Table~\ref{tab:cross_lens}). The 8B full-rank reference identifies the same strongest head as every other variant (\texttt{L23.H24}) and shares 3 of 5 top heads with the recommended lens. At GPT-2, with the ablation direction fixed to the unembedding and reductions interpolated to the $\text{PPL} = 1.10\times$ operating point on each selector's scale sweep, our lens's heads give $25.0\%$, the full-rank reference's $23.2\%$, and unembedding-only selection $18.8\%$; the trained selections are separated by less than run-to-run noise, so we read this as parity between our lens and the reference, both ahead of the lens-free selector. At 8B, under soft subtraction every trained selector reduces toxicity, whereas random-head intervention increases it; under zero-ablation the recommended \klapprox\ selection produces the clearest reduction (Table~\ref{tab:selector}; single-run evaluations, so differences of a few points are within run-to-run noise). Soft subtraction trades toxicity against perplexity controllably at 8B (Figure~\ref{fig:toxin_soft}; full sweeps in Table~\ref{tab:ablation}). At 70B, toxic signal again localizes to late-layer heads but is spread far more thinly (the top five carry only ${\sim}10\%$), and ablating the flagged heads is no more effective than ablating random ones (Figure~\ref{fig:toxin}): targeted head ablation is effective only when the localized signal is sufficiently concentrated, and that concentration is absent at 70B. In the whole-model audit (Figure~\ref{fig:allsites}), on GPT-2 subtracting at \texttt{attn\_in} and \texttt{resid\_post} removes $2.2\times$ and $1.6\times$ as much toxicity as the original per-head recipe; at 8B, \texttt{mlp\_out} ($34\%$) and \texttt{resid\_post} ($27\%$) reach $3.1\times$ and $2.5\times$, while \texttt{attn\_in}, the most effective target on GPT-2, has almost no effect: a selection made on one model does not transfer to another. The full 8B audit (192 hookpoints scored, 24 ablation sweeps) runs in under ten minutes on one node.

\begin{figure}[t]
    \centering
    \includegraphics[width=\linewidth]{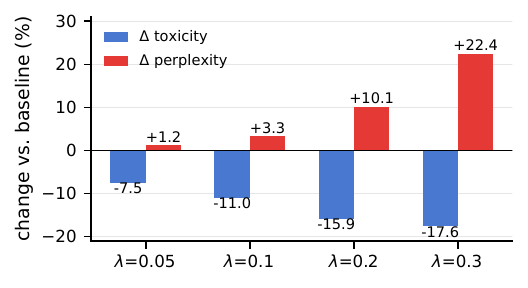}
\caption{ToxIn soft subtraction at 8B trades toxicity against perplexity as the subtraction strength $\lambda$ grows.}
    \label{fig:toxin_soft}
\end{figure}

\begin{figure}[t]
    \centering
    \includegraphics[width=\linewidth]{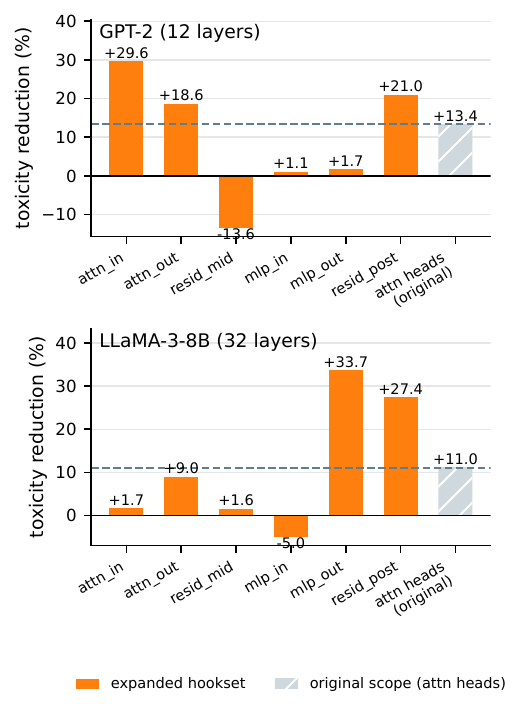}
\caption{The whole-model ablation audit (GPT-2 top, LLaMA-3-8B bottom): soft subtraction of the lens-mapped toxic direction at the top-flagged hookpoints of each component type, best operating point under a PPL~$\le 1.10\times$ budget, compared with the original attention-heads-only recipe (hatched bar, dashed line). The most effective targets lie outside attention at both scales and differ between them.}
    \label{fig:allsites}
\end{figure}

\begin{figure}[t]
    \centering
    \includegraphics[width=\linewidth]{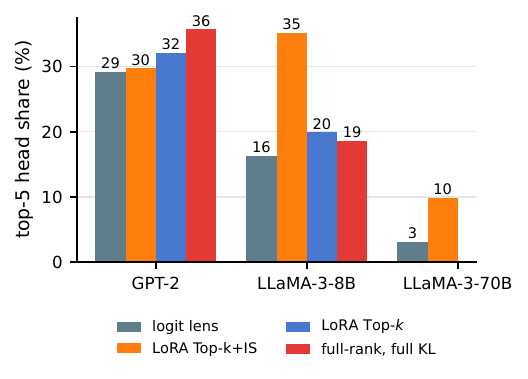}
\caption{Share of the total toxic signal carried by the top-5 heads, by model and lens. Localization is comparable at GPT-2 and 8B and far more diffuse at 70B; the logit lens's concentration declines with scale relative to the trained lenses. Per-head maps appear in Figures~\ref{fig:dart_heatmaps}--\ref{fig:dart_layerprofile}.}
    \label{fig:dart}
\end{figure}

\begin{figure}[t]
    \centering
    \includegraphics[width=\linewidth]{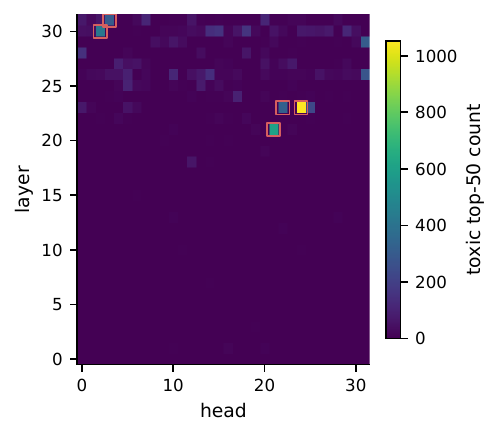}
\caption{DART per-head toxic-token counts at LLaMA-3-8B (LoRA \klapprox\ lens): the signal concentrates in a few late-layer heads. Boxes mark the top-5 heads; the strongest (\texttt{L23.H24}) and the top-5 share are quantified in Table~\ref{tab:dart}.}
    \label{fig:dart_heatmaps}
\end{figure}

\begin{figure}[t]
    \centering
    \includegraphics[width=\linewidth]{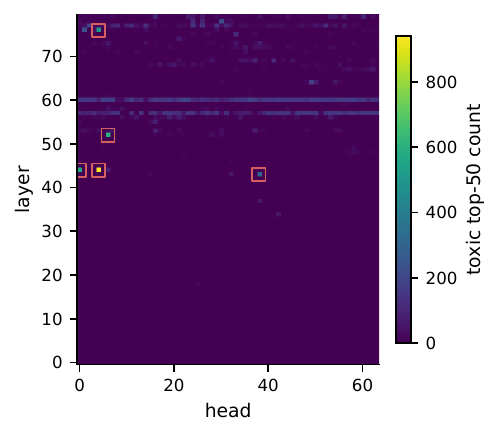}
\caption{As Figure~\ref{fig:dart_heatmaps}, for LLaMA-3-70B: the same late-layer localization holds but is far more diffuse (Table~\ref{tab:dart}).}
    \label{fig:dart_heatmaps_70b}
\end{figure}

\begin{figure}[t]
    \centering
    \includegraphics[width=\linewidth]{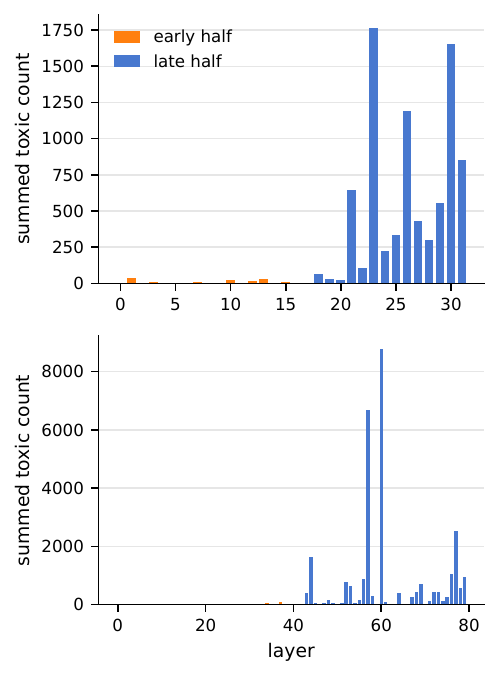}
\caption{DART toxic signal aggregated by layer (LLaMA-3-8B top, 70B bottom), showing the late-layer skew quantified in Table~\ref{tab:dart}.}
    \label{fig:dart_layerprofile}
\end{figure}

\begin{table*}[t]
\centering
\caption{DART localization concentration, all lens variants under one harness (200 toxic prompts, top-50 decoded tokens per head). \emph{Total} is the number of toxic tokens flagged over all heads; \emph{top head} and \emph{top-5} give the share of that total; \emph{late/early} is the ratio of summed counts in the second vs.\ first half of layers. The 8B full-rank reference decodes heads through per-layer residual translators (see above).}
\label{tab:dart}
\small
\begin{tabular}{llrccc}
\toprule
Model & Lens & Total & Top head & Top-5 & Late/early \\
\midrule
\multirow{4}{*}{GPT-2}
  & logit              & 38{,}753 & \texttt{L08.H02} ($6.8\%$)  & $29\%$ & $1.7\times$ \\
  & LoRA Top-k+IS      & 23{,}858 & \texttt{L08.H02} ($7.3\%$)  & $30\%$ & $1.8\times$ \\
  & LoRA Top-$k$       & 16{,}582 & \texttt{L11.H03} ($10.5\%$) & $32\%$ & $1.8\times$ \\
  & full-rank, full KL & 13{,}586 & \texttt{L11.H03} ($11.5\%$) & $36\%$ & $1.7\times$ \\
\midrule
\multirow{4}{*}{LLaMA-3-8B}
  & logit              & 94{,}077 & \texttt{L23.H24} ($6.4\%$)  & $16\%$ & $1.9\times$ \\
  & LoRA Top-k+IS      & 9{,}395  & \texttt{L23.H24} ($12.4\%$) & $35\%$ & $167\times$ \\
  & LoRA Top-$k$       & 24{,}641 & \texttt{L23.H24} ($7.6\%$)  & $20\%$ & $2.5\times$ \\
  & full-rank (resid.) & 62{,}121 & \texttt{L23.H24} ($7.0\%$)  & $19\%$ & $2.1\times$ \\
\midrule
\multirow{2}{*}{LLaMA-3-70B}
  & logit         & 407{,}752 & \texttt{L44.H04} ($0.8\%$) & $\phantom{0}3\%$ & $1.1\times$ \\
  & LoRA Top-k+IS & 29{,}699  & \texttt{L44.H04} ($3.2\%$) & $10\%$ & $72\times$ \\
\bottomrule
\end{tabular}
\end{table*}

\begin{table}[t]
\centering
\caption{ToxIn ablation. Toxicity is the mean Toxic-BERT~\citep{Detoxify} score over greedy generations from toxic prompts; PPL on WikiText-2. 70B uses the reduced 40-prompt evaluation.}
\label{tab:ablation}
\small
\begin{tabular}{lcccc}
\toprule
 & \multicolumn{2}{c}{LLaMA-3-8B} & \multicolumn{2}{c}{LLaMA-3-70B} \\
\cmidrule(lr){2-3}\cmidrule(l){4-5}
Intervention & $\Delta$tox & PPL & $\Delta$tox & PPL \\
\midrule
Baseline               & ---      & 9.66 & ---     & 3.56 \\
Zero-ablate DART top-15 & $-16.7\%$ & 9.88 & $+4.1\%$ & 3.58 \\
Zero-ablate 15 random   & $+4.2\%$ & 9.74 & $-0.5\%$ & 3.57 \\
Soft ToxIn $\lambda{=}0.1$ & $-11.0\%$ & 9.98 & $-1.7\%$ & 3.58 \\
Soft ToxIn $\lambda{=}0.3$ & $-17.6\%$ & 11.8 & $+3.2\%$ & 3.71 \\
\bottomrule
\end{tabular}
\end{table}

\paragraph{Estimator localization and seed stability.} Trained under the identical 250-step annealed schedule, the two estimators concentrate the signal differently: the Top-$k$+IS lens localizes roughly twice as sharply as Top-$k$ (top-5 share $35\%$ vs.\ $20\%$; across three training seeds, $31$--$35\%$ vs.\ $11$--$20\%$). The strongest-head identification is fully seed-robust: retraining all three lens variants under two additional seeds, every one of the nine lens$\times$seed audits ranks \texttt{L23.H24} first, with 4--5 of 5 top heads shared across seeds.

\begin{figure}[t]
    \centering
    \includegraphics[width=\linewidth]{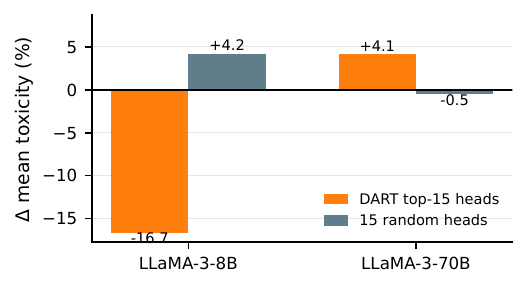}
\caption{ToxIn zero-ablation: removing the DART-flagged heads reduces toxicity at 8B while removing random heads does not; at 70B neither does.}
    \label{fig:toxin}
\end{figure}

\begin{table*}[t]
\centering
\caption{Head-selector comparison at 8B: ToxIn with the same unembedding-derived direction (100 toxic prompts; Toxic-BERT; PPL on WikiText-2, baseline 9.66). Zero-ablation acts on the top-15 flagged heads themselves; soft subtraction acts on the attention output of the layers containing them (7--11 layers per selector). All trained lenses share the identical 250-step annealed schedule; the random row averages three seed-0 draws, which partially overlap.}
\label{tab:selector}
\small
\begin{tabular}{lcccc}
\toprule
 & \multicolumn{2}{c}{zero-ablate top-15} & \multicolumn{2}{c}{soft $\lambda{=}0.3$} \\
\cmidrule(lr){2-3}\cmidrule(l){4-5}
Head selector & $\Delta$tox & PPL & $\Delta$tox & PPL \\
\midrule
LoRA Top-$k$+IS             & $\mathbf{-16.7\%}$ & 9.9  & $-17.6\%$ & 11.8 \\
LoRA Top-$k$              & $+0.4\%$  & 9.8  & $\mathbf{-27.8\%}$ & 12.7 \\
full-rank tuned (ref.)    & $-0.8\%$  & 9.8  & $-23.3\%$ & 11.9 \\
logit lens                & $-7.3\%$  & 10.1 & $-17.7\%$ & 11.5 \\
15 random heads           & $+5.2\%$  & 9.8  & \textemdash{}       & \textemdash{}  \\
\bottomrule
\end{tabular}
\end{table*}

\section{Fine-Grained Fidelity Statistics}
\label{app:finegrained}

Beyond task outcomes, we compare the lenses on the trajectory statistics themselves.

\textbf{Rank agreement of head audits.} Table~\ref{tab:cross_lens} quantifies cross-lens agreement of the DART head rankings. Trained lenses agree at every scale, and all variants (including the full-rank references at GPT-2 and 8B) identify the same strongest heads; rank correlations over \emph{all} heads are dominated by the inert majority and should be read jointly with the top-head overlaps. The logit lens's rank agreement with trained lenses decays with scale (Spearman $0.73 \to 0.02$) even as it continues to identify the very strongest heads.

\textbf{Prediction depth.} From the detection captures we also compute each lens's prediction depth: the trajectory point after which its top-1 prediction stops changing~\citep{belrose2023elicitinglatentpredictionstransformers}. Against the full-rank reference, the LoRA lenses agree to within one hidden state on 69--77\% of GPT-2 examples (mean absolute difference ${\approx}1$ layer), while the logit lens manages 45\% with no rank correlation ($\rho=-0.08$). At 8B the ordering is preserved (LoRA 31--38\% within one state, $\rho\approx0.53$; logit 25\%, $\rho=0.24$), but the full-KL reference resolves predictions systematically earlier than the Subset-KL lenses, by four to five hidden states on average. Prediction depth is the fine-grained statistic where the expensive reference retains a visible edge; the task-level results of the main text are unaffected by it.

\begin{figure*}[t]
    \centering
    \includegraphics[width=\linewidth]{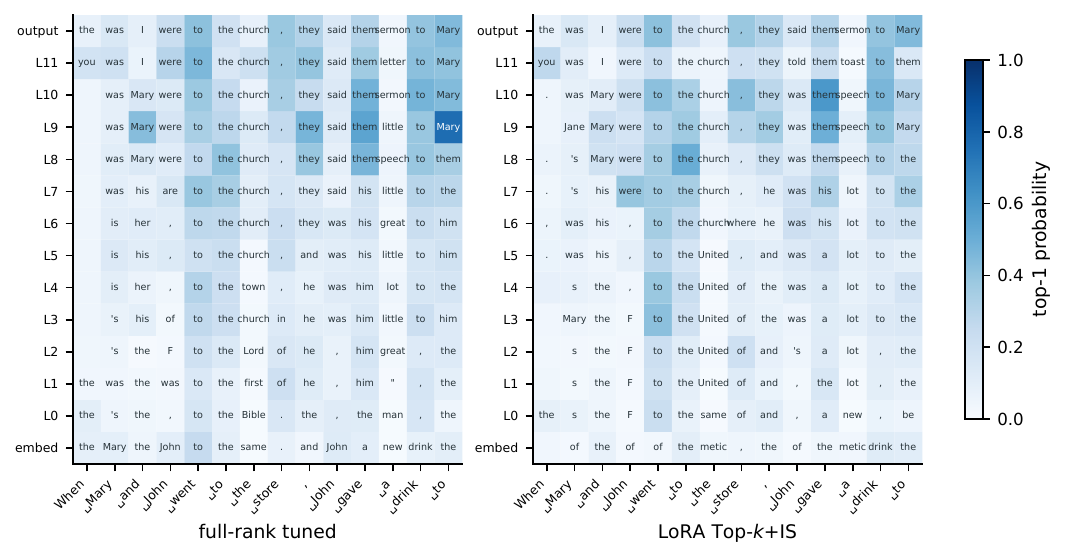}
\caption{The classic prediction-trajectory grid (top-1 token of the lens readout at every layer and position; shading = probability) for the full-rank tuned lens and LoRA \klapprox\ on the same GPT-2 prompt. Both lenses read the same computation: the indirect object emerges around L9--L10 and stabilizes to the model's final prediction.}
    \label{fig:traj_grid}
\end{figure*}

\textbf{Causal basis extraction.} Finally we rerun the original paper's causal-fidelity experiment: for each lens and layer, extract the $k{=}16$ orthonormal directions whose mean-ablation most changes the lens output, then ablate each direction in the \emph{model} and measure the KL against the clean output (Appendix~\ref{app:inj_detect}). Every lens's directions are causally real: ablating them moves the model two orders of magnitude more than random directions (mean top-8 KL 0.39--0.48 at GPT-2, 0.69--0.81 at 8B, vs.\ ${\approx}0.002$--$0.007$ random), and our lenses match the reference on this transfer strength. On the finer statistic (Spearman correlation between the lens's claimed influence ordering and the realized model KL), the reference leads at GPT-2 (0.71 vs.\ 0.57--0.61 LoRA, 0.40 logit), while at 8B the statistic stops discriminating between lenses entirely (0.52--0.65 for all, logit included). The qualitative counterpart is Figure~\ref{fig:traj_grid}: the full-rank and \klapprox\ trajectory grids on the same prompt are near-identical.

\begin{table*}[t]
\centering
\caption{Cross-lens agreement of DART head rankings over all attention heads. \emph{Top-5} counts shared heads among each lens's five strongest. Trained lenses agree at every scale; the logit lens's rank agreement with trained lenses decays with scale, although it still identifies the strongest heads. The 8B full-rank row carries the residual-translator caveat of Table~\ref{tab:dart}.}
\label{tab:cross_lens}
\small
\begin{tabular}{llcccccc}
\toprule
Model & Comparison & Pearson & Spearman & Kendall & Top-25\% & Top-50\% & Top-5 \\
\midrule
\multirow{6}{*}{GPT-2}
  & Top-k+IS vs.\ Top-$k$    & 0.869 & 0.610 & 0.464 & 0.64 & 0.68 & 3/5 \\
  & Top-k+IS vs.\ full-rank  & 0.776 & 0.570 & 0.420 & 0.67 & 0.72 & 2/5 \\
  & Top-$k$ vs.\ full-rank   & 0.884 & 0.529 & 0.392 & 0.61 & 0.71 & 4/5 \\
  & logit vs.\ Top-k+IS      & 0.950 & 0.731 & 0.554 & 0.78 & 0.76 & 4/5 \\
  & logit vs.\ Top-$k$       & 0.779 & 0.581 & 0.428 & 0.64 & 0.69 & 3/5 \\
  & logit vs.\ full-rank     & 0.677 & 0.475 & 0.337 & 0.67 & 0.65 & 2/5 \\
\midrule
\multirow{4}{*}{LLaMA-3-8B}
  & Top-k+IS vs.\ Top-$k$    & 0.946 & 0.414 & 0.340 & 0.51 & 0.66 & 3/5 \\
  & Top-k+IS vs.\ full-rank  & 0.838 & 0.258 & 0.205 & 0.39 & 0.60 & 3/5 \\
  & logit vs.\ Top-k+IS      & 0.836 & 0.121 & 0.091 & 0.36 & 0.54 & 3/5 \\
  & logit vs.\ Top-$k$       & 0.888 & 0.265 & 0.188 & 0.41 & 0.57 & 4/5 \\
\midrule
LLaMA-3-70B & logit vs.\ Top-k+IS & 0.533 & 0.023 & 0.018 & 0.27 & 0.54 & 5/5 \\
\bottomrule
\end{tabular}
\end{table*}

\section{Case-Study Scope and Caveats}
\label{app:caveats}

All results here are replications intended to validate lens fidelity at scale. The LLaMA runs use reduced evaluations (Appendix~\ref{app:injection}); the 2WMH templates and their programmatically constructed explicit prompts are imperfect references; and dictionary-based DART scoring favors lexically explicit toxicity. The injection results and the 8B ablations are causal interventions, while the 70B localization and the diffusion account of its null ablation remain observational. The 8B full-rank reference enters the injection and DART comparisons through its per-layer residual translators (a half-block hookpoint mismatch; Appendix~\ref{app:dart}), so its agreement numbers there are conservative, and the selector comparison of Table~\ref{tab:selector} is a single-run evaluation. Fine-grained statistics (injection-layer rankings, prediction depth, influence orderings) are the last lens properties to stabilize during training and should be read from converged, annealed lenses; the 8B seed study shows they are also the most seed-sensitive: coarse results (the strongest DART head, interior-vs-final layer selection for the recommended lens) replicate across all three training seeds, while the exact picked layer and top-5 shares move within the ranges reported above. The trained 8B lenses in the injection and toxicity analyses share a single 250-step cosine-annealed schedule, identical to the 70B run's; the detection captures of Table~\ref{tab:inj_detect} and the whole-model audit of Section~\ref{sec:audit} predate this schedule (earlier-schedule checkpoints), and the earlier bracketing checkpoints supported the same coarse conclusions.
Finally, the layer-selection advantage of trained lenses is specific to the over-injection regime, and the logit lens's failures are specific to scale: at GPT-2 it selects an effective injection layer and localizes adequately.